%% file: main.tex
\documentclass{article}

\usepackage[preprint, nonatbib]{nips_2018}

\usepackage[T1]{fontenc}    
\usepackage[utf8]{inputenc} 
\usepackage{nag}
\usepackage{amsfonts}       
\usepackage{amsmath}
\usepackage{amssymb}
\usepackage{amsthm}
\usepackage{booktabs}       
\usepackage{graphicx}
\usepackage{hyperref}       
\usepackage{microtype}      
\usepackage{nicefrac}       
\usepackage{url}            
\usepackage{xargs}
\usepackage{empheq}
\usepackage{algorithm, algorithmicx}
\usepackage[noend]{algpseudocode}
\usepackage[pdftex,dvipsnames]{xcolor}  
\usepackage{wrapfig}
\usepackage{subcaption}
\usepackage{tikz}
\usepackage{caption}
\usepackage{varwidth}
\usepackage{appendix}
\usepackage{setspace}
\usepackage{mdframed}
\usepackage{backnaur}
\usepackage{pgf}

\newmdtheoremenv[style=2]{mddef}{Definition}
\newmdtheoremenv[style=2]{mdtheo}{Theorem}
\newmdtheoremenv[style=2]{mdlem}{Theorem}
\newmdtheoremenv[style=2]{mdcorr}{Corollary}
\newmdtheoremenv[style=2]{mdprop}{Proposition}
\newmdtheoremenv[style=2]{mdprune}{Pruning Rule}
\newmdtheoremenv[style=proofstyle]{mdproof}{Proof}

\newenvironment{Figure}
  {\par\medskip\noindent\minipage{\linewidth}}
  {\endminipage\par\medskip}

\usepackage[colorinlistoftodos,prependcaption,textsize=tiny]{todonotes}
\newcommandx{\unsure}[2][1=]{\todo[linecolor=red,backgroundcolor=red!25,bordercolor=red,#1]{#2}}
\newcommandx{\change}[2][1=]{\todo[linecolor=blue,backgroundcolor=blue!25,bordercolor=blue,#1]{#2}}
\newcommandx{\info}[2][1=]{\todo[linecolor=OliveGreen,backgroundcolor=OliveGreen!25,bordercolor=OliveGreen,#1]{#2}}
\newcommandx{\improvement}[2][1=]{\todo[linecolor=Plum,backgroundcolor=Plum!25,bordercolor=Plum,#1]{#2}}
\newcommandx{\thiswillnotshow}[2][1=]{\todo[disable,#1]{#2}}

\usepackage{scalerel,stackengine}
\stackMath
\newcommand\reallywidehat[1]{%
\savestack{\tmpbox}{\stretchto{%
  \scaleto{%
    \scalerel*[\widthof{\ensuremath{#1}}]{\kern.1pt\mathchar"0362\kern.1pt}%
    {\rule{0ex}{\textheight}}
  }{\textheight}%
}{2.4ex}}%
\stackon[-6.9pt]{#1}{\tmpbox}%
}

\newcommand{\approptoinn}[2]{\mathrel{\vcenter{
  \offinterlineskip\halign{\hfil$##$\cr
    #1\propto\cr\noalign{\kern2pt}#1\sim\cr\noalign{\kern-2pt}}}}}

\newcommand{\appropto}{\mathpalette\approptoinn\relax}

\theoremstyle{definition}
\newtheorem{example}{Example}

\newtheorem{theorem}{Theorem}
\newtheorem{definition}{Definition}
\newtheorem{corollary}{Corollary}

\newtheorem{lemma}{Lemma}
\renewcommand\vec{\mathbf}

\input{macros/generic.macros}

\input{macros/specific.macros}

\renewcommand{\paragraph}[1]{\noindent {\bf #1}}

\mdfsetup{skipabove=5pt,skipbelow=2pt}

\title{Learning Task Specifications from Demonstrations}

\author{
  Marcell Vazquez-Chanlatte$^1$, Susmit Jha$^2$, Ashish Tiwari$^2$, Mark K. Ho$^1$, Sanjit A. Seshia$^1$\\
  $^1$ University of California, Berkeley~~~
  $^2$ SRI International, Menlo Park\\
  \{marcell.vc, sseshia, mark\_ho\}@eecs.berkeley.edu~~~
  \{susmit.jha, tiwari\}@sri.com
}

\begin{document}

\maketitle

\input{sections/abstract}
\input{sections/intro}
\input{sections/spec_inference}
\input{sections/algorithm}
\input{sections/experiments}
\input{sections/conclusion}
{
  \footnotesize
  \mypara{Acknowledgments}
  We would like to thank the anonymous referees as well as Daniel
Fremont, Markus Rabe, Ben Caulfield, Marissa Ramirez Zweiger, Shromona
Ghosh, Gil Lederman, Tommaso Dreossi, Anca Dragan, and Natarajan
Shankar for their useful suggestions and feedback.  The work of the
authors on this paper was funded in part by the US National Science
Foundation (NSF) under award numbers CNS-1750009, CNS-1740079,
CNS-1545126 (VeHICaL), the DARPA BRASS program under agreement number
FA8750--16--C0043, the DARPA Assured Autonomy program, by Toyota under
the iCyPhy center and the US ARL Cooperative Agreement
W911NF-17-2-0196.
}

\bibliography{refs}
\bibliographystyle{abbrv}

\end{document}



\maketitle
Claim:
\begin{empheq}[box=\fbox]{align}\label{eq:max_ent_dist2}
  \Prob(\vec{\xi}~|~M, \varphi)=\widehat{\vec{\xi}}\cdot
  \begin{cases}
    \overline{\varphi}/\widehat{\varphi} & \xi \in \varphi\\
    \overline{\neg \varphi}/\widehat{\neg \varphi} & \xi \notin
    \varphi
  \end{cases}
\end{empheq}
where, in general, we will use a bar over a variable, $\overline{(\cdot)}$, to indicate
average w.r.t. $X$ and a hat, $\widehat{(\cdot)}$, to denote expected value under random
actions. Thus, for example, $\widehat{\vec{\xi}}$ is the probability
of generating trace $\xi$ using the (uniformly) random action policy
and by the law of the excluded middle:
$\widehat{\neg \varphi} = 1 - \widehat{\varphi}$ and
$\overline{\neg \varphi} = 1 - \overline{\varphi}$.
  \begin{proof}
    For brevity, let
    $\big [ w_{\vec{s}\times\vec{a}} \eqdef
    \prod_{i=0}^{\tau-1}\delta(s_i, a_i, s_{i+1}) \big ]$ and
    $\big [ W_\varphi \eqdef \sum_{\vec{\xi} \in \varphi}
    w_{\vec{\xi}} \big ]$. By assumption,

    \begin{equation}\label{eq:max_ent_dist} \Prob(\vec{\xi}=(\vec{s}, \vec{a})~|~M, \varphi)
  = \frac{\exp(\lambda_\varphi \varphi(\xi))}{Z_\varphi}\prod_{i=0}^{\tau-1}\Prob(s_{i+1}~|~s_i, a_i, M)
\end{equation}
where $\vec{s}$ and $\vec{a}$ are the sequences and actions of $\xi$ respectively and $\lambda_\varphi$, $Z_\varphi$ are normalization factors such that $\Expected_\xi[\varphi] = \overline{\varphi}$ and $\displaystyle \sum_\xi \Prob(\xi~|~M,\varphi) =1$.
and thus,
    \begin{equation}
      Z_\varphi\cdot \overline{\varphi} = \sum_{\vec{\xi} \in \varphi} \varphi(\vec{\xi})e^{\lambda_\varphi} w_\xi
      = e^{\lambda_\varphi}W_\varphi;~~~Z_\varphi = e^{\lambda_\varphi}\sum_{\vec{\xi} \in
        \varphi} w_{\vec{\xi}} + e^0 \sum_{\vec{\xi} \notin \varphi} w_{\vec{\xi}} = e^{\lambda_\varphi} W_\varphi + W_{\neg \varphi}
    \end{equation}
    Combining gives
    $Z_\varphi = W_{\neg \varphi}/\overline{\neg \varphi}$. Next,
    observe that if $\vec{\xi} \not \in \varphi$, then
    $e^{\lambda_\varphi \varphi(\xi)} = 1$ and,
    \begin{equation} \Prob(\vec{\xi}~|~\varphi, M, \vec{x} \notin
      \varphi) = w_\xi(\overline{\neg \varphi}/W_{\neg \varphi})
    \end{equation}
    If $\vec{\xi} \in \varphi$ (implying $W_\varphi \neq 0$) then
    $e^{\lambda_\varphi} = Z_\varphi (\overline{\varphi} / W_\varphi)$ and
    thus,
    \begin{equation} \Prob(\vec{\xi}~|~\varphi, M, \vec{\xi} \in
      \varphi) = w_\xi (\overline{\varphi}/W_\varphi)
    \end{equation}
    Finally, observe that
    $\tilde{\varphi} = \frac{W_\varphi}{W_{\text{true}}}$ and
    $\tilde{\setof{\vec{\xi}}} =
    \frac{w_\vec{\xi}}{W_{\text{true}}}$. Substituting and factoring
    yields~\eqref{eq:max_ent_dist2}.
  \end{proof}

  \newpage

\begin{lemma}\label{lemma:G_monotone}
  $\forall \varphi', \varphi \in \Phi~.~\varphi' \subseteq \varphi$
  implies $\widehat{\varphi}' \leq \widehat\varphi$ and
  $\overline{\varphi}' \leq \overline{\varphi}$.
\end{lemma}
\begin{proof}
  Notice that: $N_{\phi'} \leq N_\phi$. Thus since
  $\overline{\varphi} \propto N_\varphi$,
  $\overline{\varphi}' \leq \overline{\varphi}$.

  Next, for brevity, again let
  $\big [ w_{\vec{s}\times\vec{a}} \eqdef
  \prod_{i=0}^{\tau-1}\delta(s_i, a_i, s_{i+1}) \big ]$ and
  $\big [ W_\varphi \eqdef \sum_{\vec{\xi} \in \varphi} w_{\vec{\xi}}
  \big ]$.

  Observing that
  $ \widehat{\phi} = (W_{\phi'} + \sum_{\vec{x}\in \phi \setminus
    \phi'} w_{\vec{x}})/ W_\top$ and $ w_{\vec{x}} \geq 0 $ it follows
  that $G_{\phi'} \leq G_\phi $.$\blacksquare$
\end{proof}

\begin{lemma}\label{lemma:local_minima}
  For all $i \in \setof{0, \ldots, |\Demos|}$, $f_i(x)$ is convex and
  $f_i(i / |\Demos|)$ is a global minimum.\footnote{Follows from
    the convexity of the KL Divergence on Bernoulli distributions.}
\end{lemma}
\begin{proof}
  In the following, take the $Beta$ and $\gamma$ terms in $f_i$ to be
  zero. This is w.l.o.g. as they only depend on $i$ which is assumed
  fixed.
  
  First, observe that if $x = i/ |x|$,
  $f_i(x) = 1$. Next, observe that:
  For $x \in (0, 1)$:
  \begin{equation} \delta_{x} f_i = f_i(x)\frac{|\Demos|x -
      i}{x(1 - x)}
  \end{equation}
  Further, notice  that $f_i(x)$ and $x(1 - x)$ are positive, and thus
  $\delta_{x} f_i \geq 0 \iff |\Demos|x - i \geq 0$.  

  Thus, the sign of $\delta_{x} f_i$ is determined
  by the sign of $x - i/|\Demos|$. Therefore, as $x$
  moves from less than $i/|\Demos|$ to larger than $i/|\Demos|$, $f_i(x)$
  decreases and then increases. Thus $x = i/|\Demos|$ is a local
  minimum. As the sign of the expression $x - i/|\Demos|$ can only
  change once. Thus, $f_i$ is convex and $f_i(i/|\Demos|)$ must be a global
  minimum. $\blacksquare$
\end{proof}

\mypara{Linear Temporal Logic}
Linear Temporal Logic extends propositional logic with temporal
operators: ``always''( denoted as $\square$), ``eventually'' (denoted
as $\diamond$), ``weak until'' (denoted as $W$), and ``next'' (denoted
as $X$). $\square \phi$ is true on a trace $x$ iff for every suffice
of $x$, $\phi$ holds. Similarly, $\diamond \phi$ is true on a trace
iff for every suffix, there exists a suffix where $\phi$ holds true.
$\phi W \psi$ holds iff $\phi$ is true until $\psi$ holds or $\phi$
always holds true. Finally, $X \phi$ is true iff $\phi$ is true on the
suffix starting at the next time step.
\newpage

\newpage
\bibliography{refs}
\bibliographystyle{abbrv}

%% file: macros/generic.macros.tex
\DeclareMathAlphabet{\mathpzc}{OT1}{pzc}{m}{it}
\newcommand{\eqdef}{\mathrel{\stackrel{\makebox[0pt]{\mbox{\normalfont\tiny def}}}{=}}}

\newcommand{\ignore}[1]{}

\newcommand{\setof}[1]{\ensuremath{\left\{#1\right\}}}

\DeclareMathOperator*{\argmax}{arg\,max}

\newcommand{\mypara}[1]{\noindent{\bf #1.}}

\newcommand{\Nat}{\mathbb{N}}

\newcommand{\Prob}{\Pr}
\newcommand{\Expected}{\mathop{\mathbb{E}}}

\newcommand{\true}{\mathit{true}}
\newcommand{\false}{\mathit{false}}

\newcounter{myctr}



%% file: macros/specific.macros.tex
\newcommand{\pOrder}{\trianglelefteq}

\newcommand{\Bernoulli}{\mathcal{B}}
\newcommand{\KL}{D_{KL}}
\newcommand{\Demos}{X}

%% file: sections/abstract.tex
\begin{abstract}
  Real-world applications often naturally decompose into several
  sub-tasks. In many settings (e.g., robotics) demonstrations provide
  a natural way to specify the sub-tasks. However, most methods for
  learning from demonstrations either do not provide guarantees that
  the artifacts learned for the sub-tasks can be safely recombined or
  limit the types of composition available.  Motivated by this
  deficit, we consider the problem of inferring Boolean non-Markovian
  rewards (also known as logical trace properties or
  \emph{specifications}) from demonstrations provided by an agent
  operating in an uncertain, stochastic environment. Crucially,
  specifications admit well-defined composition rules that are
  typically easy to interpret.  In this paper, we formulate the
  specification inference task as a maximum a posteriori (MAP)
  probability inference problem, apply the principle of maximum
  entropy to derive an analytic demonstration likelihood model and
  give an efficient approach to search for the most likely
  specification in a large candidate pool of specifications. In our
  experiments, we demonstrate how learning specifications can help
  avoid common problems that often arise due to ad-hoc reward composition.
\end{abstract}


%% file: sections/intro.tex
\section{Introduction}
In many settings (e.g., robotics) demonstrations provide a natural way
to specify a task. For example, an agent (e.g., human expert) gives
one or more demonstrations of the task from which we seek to
automatically synthesize a policy for the robot to execute.
Typically, one models the demonstrator as episodically operating
within a dynamical system whose transition relation only depends on
the current state and action (called the Markov
condition). However, even if the dynamics are Markovian, many
problems are naturally modeled in non-Markovian terms (see
Ex~\ref{ex:intro}).\vspace{4px}
\fbox{
\begin{minipage}[c]{0.97\textwidth}
  \noindent
  \begin{minipage}{0.7\textwidth}
    \begin{example}\label{ex:intro}
      Consider the task illustrated in Figure~\ref{fig:gridworld}. In
      this task, the agent moves in a discrete gridworld and can take
      actions to move in the cardinal directions (north, south, east,
      west). Further, the agent can sense abstract features of the
      domain represented as colors. The task is to reach any of the
      yellow (recharge) tiles without touching a red tile (lava) -- we
      refer to this sub-task as YR. Additionally, if a blue tile
      (water) is stepped on, the agent must step on a brown tile
      (drying tile) before going to a yellow tile -- we refer to this
      sub-task as BBY.  The last constraint requires recall of two
      state bits of history (and is thus not Markovian): one bit for whether
      the robot is wet and another bit encoding if the robot recharged
      while wet.
    \end{example}
  \end{minipage}
  \hfill
  \begin{minipage}[c]{0.25\textwidth}
    \begin{Figure}
      \centering
      \includegraphics[width=\linewidth]{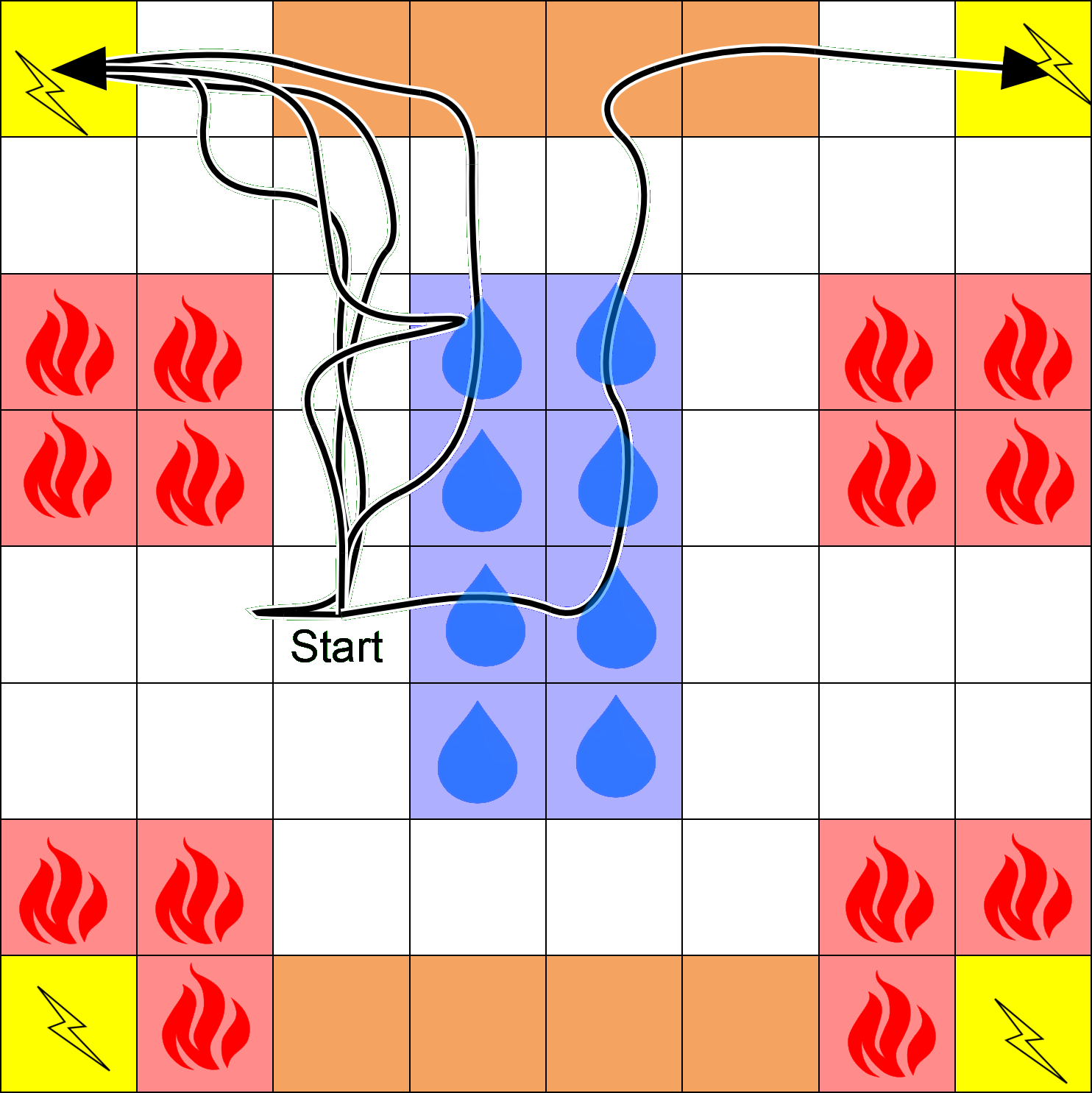}
      \captionof{figure}{\label{fig:gridworld}}
    \end{Figure}
  \end{minipage}
\end{minipage}
}

Further, like Ex~\ref{ex:intro}, many tasks are naturally decomposed
into several sub-tasks. This work aims to address the question of how
to systematically and separately learn non-Markovian sub-tasks such
that they can be readily and safely recomposed into the larger
meta-task.

Here, we argue that \textit{non-Markovian Boolean specifications} provide
a powerful, flexible, and easily transferable formalism for task
representations when learning from demonstrations. This stands in
contrast to the quantitative scalar reward functions commonly 
associated with Markov Decision Processes. 
Focusing on Boolean specifications has certain benefits: 
(1) The ability to naturally express tasks with \textit{temporal dependencies};
(2) the ability to take advantage of the \textit{compositionality} 
present in many problems, and
(3) use of {\em formal methods} for planning and verification~\cite{seshia-arxiv16}.

Although standard quantitative scalar reward functions could be used
to learn this task from demonstrations, three issues arise.  First,
consider the problem of \textit{temporal specifications}: reward
functions are typically Markovian, so requirements like those in
Ex~\ref{ex:intro} cannot be directly expressed in the task
representation. One could explicitly encode time into a state and
reduce the problem to learning a Markovian reward on new
time-dependent dynamics; however, in general, such a reduction suffers
from an exponential blow up in the state size (commonly known as the
\textit{curse of history} \cite{pineau2003point}).  When inferring
tasks from demonstrations, where different hypotheses may have
different historical dependencies, na\"{i}vely encoding the entire
history quickly becomes intractable.

A second limitation relates to the \textit{compositionality} of task
representations. As suggested, Ex~\ref{ex:intro} naturally
decomposes into two sub-tasks. Ideally, we would want an algorithm
that could learn each sub-task and combine them into the complete
task, rather than only be able to learn single monolithic tasks.
However, for many classes of quantitative rewards, "combining" rewards
remains an ad-hoc procedure. The situation is further exacerbated by
humans being notoriously bad at anticipating or identifying when
quantitative rewards will lead to unintended consequences
\cite{ho_teaching_2015}, which poses a serious problem for AI
safety~\cite{amodei2016concrete} and has led to investigations into
reward repair~\cite{ghosh2018repair}. For instance, we could take a
linear combination of rewards for each of the subtasks in
Ex~\ref{ex:intro}, but depending on the relative scales of the
rewards, and temporal discount rate, wildly different behaviors would
result.

\begin{wrapfigure}{r}{4cm}
  \centering
  \begin{subfigure}{1.0\linewidth}
    \centering
    \def\svgwidth{3.2cm}
    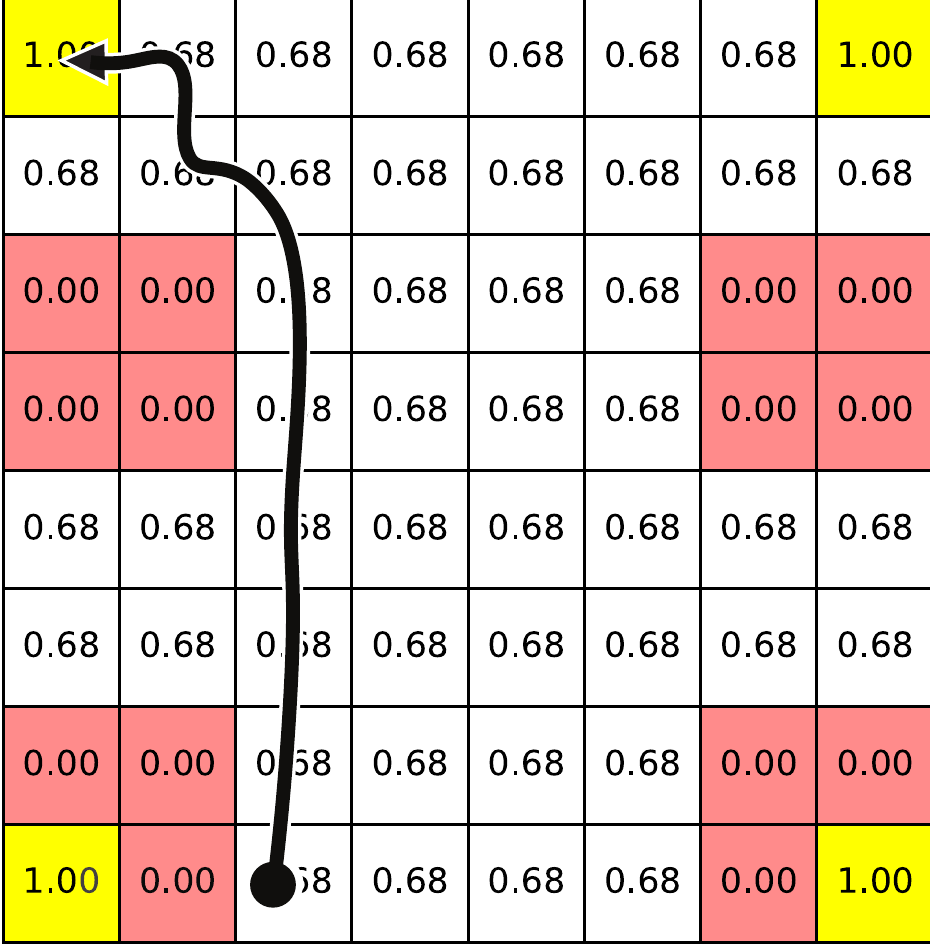
    \caption{\label{fig:yr}}
  \end{subfigure}
  \begin{subfigure}{1.0\linewidth}
    \centering
    \def\svgwidth{3.2cm}
    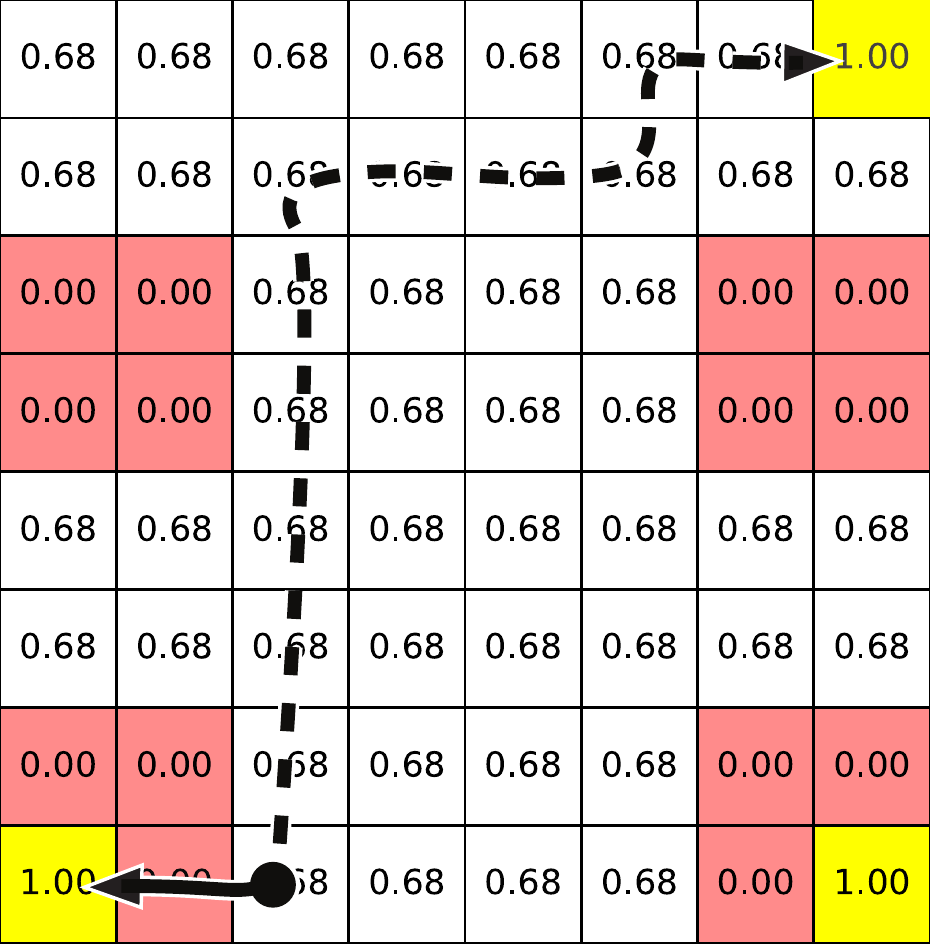
    \caption{\label{fig:bug}}
  \end{subfigure}
  \caption{Illustration of a bug in the learnt quantitative Markovian
    reward resulting from slight changes in the
    environment.\label{fig:brittle}}
\end{wrapfigure}
In fact, the third limitation - brittleness due to simple changes in
the environment - illustrates that often, the correctness of the agent
can change due to a simple change in the environment.  Namely, imagine
for a moment we remove the water and drying tiles from
Fig~\ref{fig:gridworld} and attempt to learn a reward that encodes the
``recharge while avoid lava'' task in
Ex~\ref{ex:intro}. Fig~\ref{fig:yr} illustrates the reward resulting
from performing Maximum Entropy Inverse Reinforcement
Learning~\cite{ziebart2008maximum} with the demonstrations shown in
Fig~\ref{fig:gridworld} and the binary features: red (lava tile),
yellow (recharge tile), and ``is wet''.  As is easy to verify, a
reward optimizing agent, $\sum_{i=0}^\infty \gamma^i r_i(s)$, with a
discount factor of $\gamma=0.69$ would generate the trajectory shown
in Fig~\ref{fig:yr} which avoids lava and eventually recharges.

Unfortunately, using the \emph{same} reward and discount factor on a
nearly identical world can result in the agent entering the lava. For
example, Fig~\ref{fig:bug} illustrates the learned reward being
applied to a change in the gridworld where the top left charging tile
has been removed. An acceptable trajectory is indicated via a dashed
arrow. Observe that now the discounted sum of rewards is maximized on
the solid arrow's path, resulting in the agent entering the lava! While
it is possible to find new discount factors to avoid this behavior,
such a supervised process would go against the spirit of automatically
learning the task.

Finally, we briefly remark that while non-Markovian Boolean rewards
cannot encode all possible rewards, e.g., ``run as fast as possible'',
often times such objectives can be reframed as policies for a Boolean
task. For example, consider modeling a race. If at each time step
there is a non-zero probability of entering a losing state, the agent
will run forward as fast as possible even for the Boolean task ``win
the race''.

Thus, quantitative Markovian rewards are limited as 
a task representation when learning tasks containing
\textit{temporal specifications} or \textit{compositionality} from
demonstrations.  Moreover, the need to fine tune learned tasks with
such properties seemingly undercuts the original purpose of learning
task representations that are generalizable and invariant to irrelevant
aspects of a task \cite{littman_environment-independent_2017}.


\textbf{Related Work:}
Our work is intimately related to Maximum Entropy
Inverse Reinforcement Learning. In Inverse Reinforcement Learning
(IRL)~\cite{ng2000algorithms} the demonstrator, operating in a
stochastic environment, is assumed to attempt to (approximately)
optimize some unknown reward function over the trajectories. In
particular, one traditionally assumes a trajectory's reward is the sum
of state rewards of the trajectory. This formalism offers a
succinct mechanism to encode and generalize the goals of the
demonstrator to new and unseen environments.

In the IRL framework, the problem of learning from demonstrations can
then be cast as a Bayesian inference
problem~\cite{ramachandran2007bayesian} to predict the most probable
reward function. To make this inference procedure well-defined and
robust to demonstration/modeling noise, Maximum Entropy
IRL~\cite{ziebart2008maximum} appeals to the principle of maximum
entropy~\cite{jaynes1957information}. This results in a likelihood
over the demonstrations which is no more committed to any particular
behavior than what is required for matching the empirically observed
reward expectation. While this approach was initially limited to
learning a linear combination of  feature vectors, IRL has been
successfully adapted to arbitrary function approximators such as
Gaussian processes~\cite{NIPS2011_4420} and neural
networks~\cite{finn2016guided}. As stated in the introduction,
while powerful, traditional IRL provides no principled mechanism
for composing the resulting rewards.

To address this deficit, composition using soft optimality has recently
received a fair amount of attention; however, the compositions are
limited to either strict disjunction (do X \emph{or}
Y)~\cite{todorov2007linearly}~\cite{todorov2008general} or conjunction
(do X \emph{and} Y)~\cite{haarnoja2018composable}. Further, because soft
optimality only bounds the deviation from simultaneously optimizing
both rewards, optimizing the composition does not preclude violating
safety constraints embedded in the rewards (e.g., do not enter the
lava).

The closest work to ours is recent work on inferring Linear Temporal
Logic (LTL) by finding the specification that minimizes the expected
number of violations by an optimal agent compared to the expected
number of violations by an agent applying actions uniformly at
random~\cite{kasenberg2017interpretable}. The computation of the
optimal agent's expected violations is done via dynamic programming on
the explicit product of the deterministic Rabin
automaton~\cite{farwer2002omega} of the specification and the state
dynamics. A fundamental drawback to this procedure is that due to the
curse of history, it incurs a heavy run-time cost, even on simple two
state and two action Markov Decision Processes. We also note that the
literature on learning logical specifications from examples (e.g.,
\cite{jha2017telex,vazquez2017logical,Li:EECS-2014-20}), 
does not handle noise in examples while our
approach does.
Finally, once a specification has been identified, one can leverage
the rich literature on planning using temporal logic to synthesize a
policy~\cite{kress2009temporal,saha2014automated,raman2015reactive,jha2016automated,jha2018safe}.

\textbf{Contributions:}
(i) We formulate the problem of \emph{learning specifications from
demonstrations} in terms of Maximum a Posteriori inference. (ii) To
make this inference well defined, we appeal to the principle of
maximum entropy culminating in the distribution given
\eqref{eq:model}. The main contribution of this model is that it only
depends on the probability the demonstrator will successfully perform
task and the probability that the task is satisfied by performing
actions uniformly at random.  Because these properties can be
estimated without explicitly unrolling the dynamics in time, this
model avoids many of the pitfalls characteristic to the curse of
history. (iii) We provide an algorithm that exploits the piece-wise
convex structure in our posterior model~\eqref{eq:model} to
efficiently perform Maximum a Posteriori inference for the most
probable specification.

\textbf{Outline:} In Sec~\ref{sec:setup}, we define specifications and
probabilistic automata (Markov Decision Processes without rewards).
In Sec~\ref{sec:inference}, we introduce the problem of
\emph{specification inference from demonstrations}, and inspired by
Maximum Entropy IRL~\cite{ziebart2008maximum}, develop a model of the
posterior probability of a specification given a sequence of
demonstrations.  In Sec~\ref{sec:alg}, we develop an algorithm
to perform inference under~\eqref{eq:model}.  Finally, in Sec~\ref{sec:casestudies}, we demonstrate
how due to their inherent composability, learning specifications can
avoid common bugs that often occur due to ad-hoc reward composition.


%% file: imgs/with-recharge.pdf_tex
\begingroup%
  \makeatletter%
  \providecommand\color[2][]{%
    \errmessage{(Inkscape) Color is used for the text in Inkscape, but the package 'color.sty' is not loaded}%
    \renewcommand\color[2][]{}%
  }%
  \providecommand\transparent[1]{%
    \errmessage{(Inkscape) Transparency is used (non-zero) for the text in Inkscape, but the package 'transparent.sty' is not loaded}%
    \renewcommand\transparent[1]{}%
  }%
  \providecommand\rotatebox[2]{#2}%
  \newcommand*\fsize{\dimexpr\f@size pt\relax}%
  \newcommand*\lineheight[1]{\fontsize{\fsize}{#1\fsize}\selectfont}%
  \ifx\svgwidth\undefined%
    \setlength{\unitlength}{267.83999634bp}%
    \ifx\svgscale\undefined%
      \relax%
    \else%
      \setlength{\unitlength}{\unitlength * \real{\svgscale}}%
    \fi%
  \else%
    \setlength{\unitlength}{\svgwidth}%
  \fi%
  \global\let\svgwidth\undefined%
  \global\let\svgscale\undefined%
  \makeatother%
  \begin{picture}(1,1.01478494)%
    \lineheight{1}%
    \setlength\tabcolsep{0pt}%
    \put(0,0){\includegraphics[width=\unitlength,page=1]{imgs/with-recharge.pdf}}%
  \end{picture}%
\endgroup%

%% file: imgs/missing-recharge.pdf_tex
\begingroup%
  \makeatletter%
  \providecommand\color[2][]{%
    \errmessage{(Inkscape) Color is used for the text in Inkscape, but the package 'color.sty' is not loaded}%
    \renewcommand\color[2][]{}%
  }%
  \providecommand\transparent[1]{%
    \errmessage{(Inkscape) Transparency is used (non-zero) for the text in Inkscape, but the package 'transparent.sty' is not loaded}%
    \renewcommand\transparent[1]{}%
  }%
  \providecommand\rotatebox[2]{#2}%
  \newcommand*\fsize{\dimexpr\f@size pt\relax}%
  \newcommand*\lineheight[1]{\fontsize{\fsize}{#1\fsize}\selectfont}%
  \ifx\svgwidth\undefined%
    \setlength{\unitlength}{267.80001068bp}%
    \ifx\svgscale\undefined%
      \relax%
    \else%
      \setlength{\unitlength}{\unitlength * \real{\svgscale}}%
    \fi%
  \else%
    \setlength{\unitlength}{\svgwidth}%
  \fi%
  \global\let\svgwidth\undefined%
  \global\let\svgscale\undefined%
  \makeatother%
  \begin{picture}(1,1.01493646)%
    \lineheight{1}%
    \setlength\tabcolsep{0pt}%
    \put(0,0){\includegraphics[width=\unitlength,page=1]{imgs/missing-recharge.pdf}}%
  \end{picture}%
\endgroup%

%% file: sections/spec_inference.tex
\vspace{-10px}
\section{Background}\label{sec:setup}
We seek to learn specifications from demonstrations provided by a
teacher who executes a sequence of actions that probabilistically
changes the system state. For simplicity, we assume that the set of
actions and states are finite and fully observed. The system is
naturally modeled as a probabilistic automaton\footnote{Probabilistic Automata are often constructed as a
Markov Decision Process, $M$, without its Markovian reward map $R$, denoted $M \setminus R$.}
formally defined below:
\begin{mddef}[Probabilistic Automaton]
  A \textbf{probabilistic automaton} is a tuple $ M = (S, s_0, A, \delta) $,
  where $ S $ is the finite set of states, $s_0 \in S$ is the initial
  state, $ A $ is the finite set of actions, and
  $ \delta : S \times A \times S \to [0, 1]$ specifies the transition
  probability of going from $ s $ to $ s' $ given action $ a $,
  i.e. $\delta(s, a, s') = \Prob(s'~|~s,a) $ and
  $\displaystyle \sum_{s'\in S} \Prob(s'~|~s,a) = 1$ for all states
  $s$.
\end{mddef}

\begin{mddef}[Trace]
  A sequence of state/action pairs is called a \textbf{trace}
  (trajectory, demonstration).  A trace of length $\tau \in \Nat$ is
  an element of $(S \times A)^{\tau}$.
\end{mddef}
Next, we develop machinery to distinguish between desirable and
undesirable traces. For simplicity, we focus on finite trace
properties, referred to as specifications, that are decidable within
some fixed $\tau \in \Nat$ time steps, e.g., ``event A occurred in the
last 20 steps''.
\begin{mddef}[Specification]
  Given a set of states $S$, a set of actions $A$, and a fixed trace
  length $\tau \in \Nat$, a \textbf{specification} is a subset of
  traces $\varphi \subseteq (S \times A)^\tau$. We define
  $\true \eqdef (S \times A)^\tau$,
  $\neg \varphi \eqdef \true \setminus \varphi$, and
  $\false \eqdef \neg \true$. A collection
  of specifications, $\Phi$, is called a \textbf{concept class}.
  Finally, we abuse notation and use
  $\varphi$ to also denote its indicator function (interpreted as a non-Markovian Boolean reward),
  \begin{equation}
    \varphi(\xi) \eqdef 
    \begin{cases}
      1 & \text{if } \xi \in \varphi\\
      0 & \text{otherwise}
    \end{cases}
    .
  \end{equation}
\end{mddef}
Specifications may be given in formal notation, as sets or
automata. Further, each representation facilitates defining a
plethora of composition rules. For example, consider two specifications,
$\varphi_A$, $\varphi_B$ that encode tasks $A$ and $B$ respectively
and the composition rule $\varphi_A \cap \varphi_B: \xi \mapsto
\min\big (\varphi_A(\xi), \varphi_B(\xi)\big )$. Because the agent
only receives a non-zero reward if $\varphi_A(\xi) = \varphi_B(\xi) = 1$, a
reward maximizing agent must necessarily perform tasks $A$ and $B$
simultaneously. Thus, $\varphi_A \cap \varphi_B$ corresponds to
conjunction (logical \emph{and}). Similarly, maximizing $\varphi_A \cup
\varphi_B : \xi \mapsto \max\big(\varphi_A(\xi), \varphi_B(\xi)\big)$
corresponds to disjunction (logical \emph{or}). One can also encode
conditional requirements using subset inclusion, e.g.,
maximizing $\varphi_A \subseteq \varphi_B : \xi \mapsto
\max\big(1 - \varphi_A(\xi), \varphi_B(\xi)\big)$ corresponds to
task A triggering task B.

Complicated temporal connectives can also be defined using temporal
logics~\cite{pnueli1977temporal} or
automata~\cite{vardi1996automata}. For our purposes, it suffices to
informally extend propositional logic with three temporal
operators: (1) Let $H a$, read ``historically $a$'', denote that property $a$ held at all previous time steps. (2) Let $P a \eqdef \neg (H
\neg a)$, read ``once $a$'', denote that the property $a$ 
at least once held in the past. (3) Let $a~S~b$, read ``$a$ since $b$'', denote that the property
$a$ that has held every time step after $b$ last held. The true power
of temporal operators is realized when they are composed to make more
complicated sentences. For example, $H (a \implies (b~S~c))$
translates to ``it was always the case that if $a$ was true, then
$b$ has held since the last time $c$ held.''. Observe that the property BBY from the introductory example
takes this form, $H ((\text{yellow} \wedge P~\text{blue}) \implies
(\neg \text{blue}~S~\text{brown}))$, i.e., ``Historically, if the agent had once visited blue and is currently visiting yellow, then the agent has not visited blue since it last visited brown''.

\section{Specification Inference from Demonstrations}\label{sec:inference}
\vspace{-5px}
In the spirit of Inverse Reinforcement Learning, we now seek to
find the specification that best explains the behavior of the
agent. We refer to this as \emph{Specification Inference from Demonstrations}.
\begin{mddef}[Specification Inference from
  Demonstrations]\label{problem1}
  The \textbf{specification inference from demonstrations} problem is a tuple
  $(M, \Demos, \Phi)$ where $M=(S, s_0, A, \delta) $ is a
  probabilistic automaton, $\Demos$ is a (multi-)set of $\tau$-length
  traces drawn from an unknown distribution induced by a teacher
  attempting to demonstrate some unknown specification within $M$, and $ \Phi $ a
  concept class of specifications. 
  
  A solution to
  $(M, \Demos, \Phi)$ is:
  \begin{equation}
    \label{eq:RiskAverseProblem} \varphi^* \in \argmax_{\varphi \in \Phi}\Prob
    (\varphi~|~M, \Demos)\\
  \end{equation}
  where $\Prob(\varphi~|~M, \Demos)$ denotes the probability that
  the teacher demonstrated $\varphi$
  given the observed traces, $\Demos$, and the dynamics, $M$.
\end{mddef}
To make this inference well-defined, we make a series of assumptions
culminating in \eqref{eq:model}.

\textbf{Likelihood of a demonstration:}
We begin by leveraging the principle of maximum entropy to
disambiguate the likelihood distributions. Concretely,
define:
\begin{equation}
  \label{eq:trace_weight}
  w\big(\vec{\xi}=(\vec{s}, \vec{a}), M\big) = \prod_{i=0}^{\tau-1}\Prob(s_{i+1}|s_i, a_i, M)
\end{equation} where $\vec{s}$ and $\vec{a}$ are the 
projected sequences of states and
actions of $\xi$ respectively,
to be the weight of each possible demonstration $\xi$ induced by
dynamics $M$.  Given a demonstrator who on average satisfies the
specification $\varphi$ with probability $\overline{\varphi}$, we
approximate the likelihood function by:
\begin{equation}\label{eq:max_ent_dist}
    \Prob\big(\vec{\xi}~|~M, \varphi, \overline{\varphi}\big) = 
    w(\xi, M)\cdot \frac{\exp(\lambda_\varphi
      \varphi(\xi))}{Z_\varphi}
\end{equation} where $\lambda_\varphi$, $Z_\varphi$ are
normalization factors such that $\Expected_\xi[\varphi] =
\overline{\varphi}$ and $\sum_\xi \Prob(\xi~|~M,\varphi)
=1$. For a detailed derivation that~\eqref{eq:max_ent_dist} is the maximal entropy
distribution, we point the reader
to~\cite{DBLP:journals/corr/abs-1805-00909}. Next observe
that due to the Boolean nature of $\varphi$,~\eqref{eq:max_ent_dist} admits a simple closed form:
\begin{empheq}[box=\fbox]{align}\label{eq:max_ent_dist2}
  \Prob(\vec{\xi}~|~M, \varphi, \overline{\varphi})=\widetilde{\setof{\vec{\xi}}}\cdot
  \begin{cases}
    \overline{\varphi}/\widetilde{\varphi} & \xi \in \varphi\\
    (1 - \overline{\varphi})/\widetilde{\neg \varphi} & \xi \notin
    \varphi
  \end{cases}
\end{empheq}
where in general we use $\widetilde{(\cdot)}$ to denote the probability of
satisfying a specification using uniformly random actions. Thus, we denote by $\widetilde{\setof{\vec{\xi}}}$ the probability of randomly generating demonstration $\xi$
within $M$. Further, note that by the law of the excluded middle, for any specification:
$\widetilde{\neg \varphi} = 1 - \widetilde{\varphi}$.

\begin{proof}[Proof Sketch]
  For brevity, let
  $W_\varphi \eqdef \sum_{\vec{\xi} \in \varphi}
  w(\vec{\xi}, M)$ and $c \eqdef e^{\lambda_\varphi}$. Via the constraints
  on~\eqref{eq:max_ent_dist},
  \begin{equation}
    \begin{split}
      &Z_\varphi\cdot \overline{\varphi} = 1\cdot\sum_{\vec{\xi} \in \varphi} c^1\cdot w(\xi, M) + 0\cdot\sum_{\vec{\xi} \notin \varphi} c^0\cdot w(\xi, M)
    = c W_\varphi\\
    &Z_\varphi = c^1 \sum_{\vec{\xi} \in
      \varphi} w(\vec{\xi}, M) + c^0\sum_{\vec{\xi} \notin \varphi} w(\vec{\xi},M) = c W_\varphi + W_{\neg \varphi}
    \end{split}
  \end{equation}
  Combining gives
  $Z_\varphi = W_{\neg \varphi}/(1 - \overline{\varphi})$. Next,
  observe that if $\vec{\xi} \not \in \varphi$, then
  $e^{\lambda_\varphi \varphi(\xi)} = 1$ and substituting in~\eqref{eq:max_ent_dist} yields,
  $\displaystyle\Prob(\vec{\xi}~|~\varphi, M, \vec{\xi} \notin
    \varphi) = w_\xi(1 - \overline{\varphi})/ W_{\neg \varphi}$.
  If $\vec{\xi} \in \varphi$ (implying $W_\varphi \neq 0$) then
  $e^{\lambda_\varphi} = Z_\varphi \overline{\varphi} / W_\varphi$ and
  $\displaystyle \Prob(\vec{\xi}~|~\varphi, M, \vec{\xi} \in
  \varphi) = w_\xi\overline{\varphi}/ W_\varphi$.
  Finally, observe that
  $\widetilde{\varphi} = W_\varphi/ W_{\text{true}}$ and
  $\widetilde{\setof{\vec{\xi}}} =
  w_\vec{\xi}/ W_{\text{true}}$. Substituting and factoring
  yields~\eqref{eq:max_ent_dist2}.
\end{proof}  

\textbf{Likelihood of a set of demonstrations:}
If the teacher gives a finite sequence of $\tau$ length
demonstrations, $\Demos$, drawn i.i.d. from~\eqref{eq:max_ent_dist2},
then the log likelihood, $\mathcal{L}$, of $\Demos$ under~\eqref{eq:max_ent_dist2} is:\footnote{ We
  have suppressed a multinomial coefficient required if any two
  demonstrations are the same.  However, this term will not
  change as $\varphi$ varies, and thus cancels when comparing across specifications.}
\begin{equation}\label{eq:max_ent_dist3}
  \mathcal{L}(\Demos~|~M, \varphi, \overline{\varphi}) = \displaystyle \log\bigg (\prod_{\vec{\xi} \in \Demos}
  \widetilde{\setof{\vec{\xi}}}\bigg) + N_\varphi \ln\left(\frac{\overline{\varphi}}{
    \widetilde{\varphi}}\right)
    + N_{\neg\varphi} \ln\left(\frac{\overline{\neg\varphi}}{
    \widetilde{\neg \varphi}}\right)
\end{equation}
where by definition we take $(0\cdot \ln( \ldots ) = 0)$ and $\displaystyle N_\varphi \eqdef \sum_{\vec{\xi}\in \Demos} \varphi(\xi)$.

Next, observe that
$\displaystyle \left[\overline{\varphi}\ln\left(\frac{\overline{\varphi}}
    {\widetilde{\varphi}}\right) + (1 - \overline{\varphi})\ln\left(\frac{1 -
      \overline{\varphi}}{ 1 - \widetilde{ \varphi}}\right)\right]$ is the
information gain (KL divergence) between two Bernoulli distributions with means
$\overline{\varphi}$ and $\widetilde{\varphi}$ resp. Syntactically, let $\mathcal{B}(\mu)$
denote a Bernoulli distribution with mean $\mu$ and
$\displaystyle \KL(P~\|~Q) \eqdef \sum_{i}P(i) \ln(P(i)/Q(i))$
denote the information gain when using distribution
$P$ compared to $Q$. If $X$ is ``representative'' such that $N_\varphi \approx \overline{\varphi}\cdot|\Demos|$, we can (up to a $\varphi$ independent normalization) approximate~\eqref{eq:max_ent_dist3}:
\begin{equation}\label{eq:max_ent_dist5}
  \Prob(\Demos~|~M, \varphi, \overline{\varphi})\appropto
  \exp{\Big(|\Demos|\cdot \KL\Big
      (\Bernoulli(\overline{\varphi})~\|~\Bernoulli(\widetilde{\varphi})\Big
)\Big)}  
\end{equation}
Where $\appropto$ denotes approximately proportional to.
Unfortunately, the approximation $|\Demos|\cdot \overline{\varphi}
\approx N_\phi$ implies that, $\overline{\neg \varphi} = 1 -
\overline{\varphi}$ which introduces the undesirable symmetry,
$\Prob(\Demos~|~M, \varphi, \overline{\varphi}) = \Prob(\Demos~|~M,
\neg \varphi, \overline{\neg \varphi})$,
into~\eqref{eq:max_ent_dist5}. To break this
symmetry, we assert that the demonstrator must
be at least as good as random. Operationally, we assert that
$\Prob(\varphi~|~\overline{\varphi} < \widetilde{\varphi}) = 0$ and is
otherwise uniform. Finally, we arrive at the posterior distribution given
in~\eqref{eq:model}, where $\mathbf{1}[\cdot]$ denotes an indicator function.
\begin{empheq}[box=\fbox]{align}\label{eq:model}
  ~~\Prob(\varphi~|~M, \Demos, \overline{\varphi}) \appropto \overbrace{\mathbf{1}[\overline{\varphi}\geq \widetilde{\varphi}]}^{\mathclap{\text{Demonstrator is better than random.}}}\cdot\exp\big(
      |\Demos|\cdot\overbrace{\KL\left
    (\Bernoulli(\overline{\varphi})~\|~\Bernoulli(\widetilde{\varphi})\right
    )}^{\mathclap{\text{Information gain over random actions.}}}
  \big)~~
\end{empheq}


%% file: sections/algorithm.tex
\section{Algorithm}\label{sec:alg}
In this section, we exploit the structure imposed
by~\eqref{eq:model} to efficiently search for the most probable
specification~\eqref{eq:RiskAverseProblem} within a (potentially large) concept class, $\Phi$. Namely, observe that under~\eqref{eq:model}, the specification inference problem~\eqref{eq:RiskAverseProblem} reduces to maximizing the information gain over random actions.
\begin{empheq}[box=\fbox]{align}\label{eq:max_kl}
    \varphi^* \in \argmax_{\varphi \in \Phi}\left \{\mathbf{1}[\overline{\varphi}\geq \widetilde{\varphi}]\cdot\KL\Big
    (\Bernoulli(\overline{\varphi})~\|~\Bernoulli(\widetilde{\varphi})\Big)\right \}
\end{empheq}

Because gradients on
$\widetilde{\varphi}$ and $\overline{\varphi}$ are not well-defined,
gradient descent based algorithms are not applicable.  Further, while
evaluating if a trace satisfies a specification is fairly efficient
(and thus our $N_\varphi / |\Demos|$ approximation to $\overline{\varphi}$ is assumed easy to compute), computing
$\widetilde{\varphi}$ is in general known to be
$\#P$-complete~\cite{bacchus2003algorithms}. 
Nevertheless, in practice, moderately efficient methods
for computing or approximating $\widetilde{\varphi}$ exist including
Monte Carlo simulation~\cite{metropolis1949monte} and weighted model
counting~\cite{chavira2008probabilistic} via Binary Decision
Diagrams (BDDs)~\cite{bryant1992symbolic} or repeated SAT
queries~\cite{chakraborty2016algorithmic}. As such, we seek an
algorithm that poses few $\widetilde{\varphi}$ queries. We  begin
with the observation that adding a trace to a
specification cannot lower its probability of satisfaction under random actions.
\begin{mdframed}
  \begin{lemma}\label{lemma:G_monotone}
    $\forall \varphi', \varphi \in
    \Phi~.~\varphi' \subseteq \varphi$ implies
    $\widetilde{\varphi}' \leq \widetilde\varphi$ and
    $\overline{\varphi}' \leq \overline{\varphi}$.
  \end{lemma}
  \begin{proof}
    \vspace{-8px}
    The probability of sampling an element of a
    set monotonically increases as elements are added to the set
    independent of the \emph{fixed} underlying distribution over elements.
  \end{proof}  
\end{mdframed}

Further, note that $N_\varphi$ (and thus, our approximation to
$\overline{\varphi}$) can only take on $|\Demos|+1$ possible
values. This suggests a piece-wise analysis of~\eqref{eq:max_kl} by
conditioning on the value of $\overline{\varphi}$.
\begin{mddef}
  Given candidate specifications $\Phi$ and a subset of demonstrations $S \subseteq \Demos$ define,
  \begin{equation}
    \Phi_S \eqdef \setof{\varphi \in \Phi~:~ \varphi \cap \Demos = S }~~~~~
      J_{|S|}(x) \eqdef \mathbf{1}\left[\frac{|S|}{|\Demos|}\geq x\right]\cdot\KL\left
    (\Bernoulli(\frac{|S|}{|\Demos|})~\|~\Bernoulli(x)\right)
  \end{equation}
\end{mddef}
The next key observation is that $J_{|S|} : [0, 1] \to \mathbb{R}_{\geq 0}$ monotonically decreases in $x$.
\begin{mdframed}
  \begin{lemma}\label{lemma:local_minima}
    $\forall S \subseteq \Demos,~x < x' \implies J_{|S|}(x) \leq J_{|S|}(x')$
  \end{lemma}
  \begin{proof}
    \vspace{-6px} To begin, observe that $\KL$ is always
    non-negative. Due to the
    $1[\frac{|S|}{|\Demos|} \geq x]$ indicator,
    $J_{|S|}(x)=0$ for all $x > |S|/|\Demos|$.  Next, observe that $J_{|S|}$
    is convex due to the convexity of the $\KL$ on Bernoulli
    distributions and is minimized at $x=|S|/|\Demos|$ (KL Divergence
    of identical distributions is $0$).  Thus, $J_{|S|}(x)$ monotonically
    decreases as $x$ increases.
  \end{proof}
\end{mdframed}
\begin{figure}[h]
  \centering
  \begin{subfigure}[c]{0.32\textwidth}
    \includegraphics[width=\textwidth]{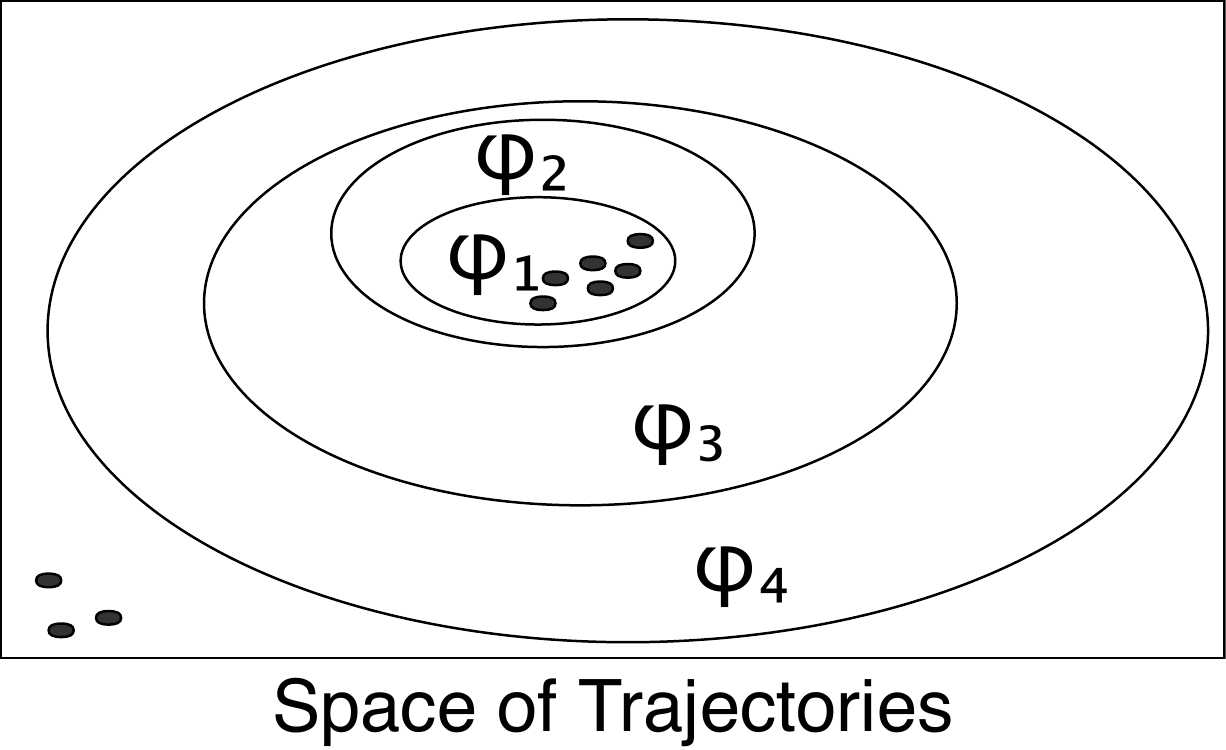}%
  \end{subfigure}
  \begin{subfigure}[c]{0.42\textwidth}
    \scalebox{0.36}{\input{imgs/thm.pgf}}
  \end{subfigure}
  \vspace{-2px}
  \caption{Left: An example of a series of
    specifications $\varphi_1, \ldots, \varphi_4$ ordered by subset
    inclusion. The dots represent demonstrations, and so each
    specification has $\overline{\varphi}_i = 6/9$. Right: Plot of
    $J_{|S|}(x)$ for hypothetical values of $\widetilde{\varphi}_i$
    annotated as points. Notice that the sequence of specifications is
    ordered on the $x$-axis, and thus the maximum must occur at the
    start of the sequence.\label{fig:thm41}}
\end{figure}
These insights are then combined in Theorem~\ref{thm:max_f_in_chain}
and illustrated in Fig~\ref{fig:thm41}.
\begin{mdframed}[style=2, everyline=true]
\begin{theorem}\label{thm:max_f_in_chain}
  If $A$ denotes a sequence of specifications,
  $\varphi_1, \ldots, \varphi_n$, ordered by subset
  inclusion $j \leq k \implies \varphi_j \subseteq
  \varphi_k$ and $S \subseteq \Demos$ is an arbitrary subset of demonstrations, then:
  \begin{equation}
    \max_{\varphi \in A} J_{|S|}(\widetilde{\varphi}) = J_{|S|}(\widetilde{\varphi}_1)    
  \end{equation}
\end{theorem}
\begin{proof}
  \vspace{-6px}
  $\widetilde{\varphi}$ is monotonically increasing on $A$
  (Lemma~\ref{lemma:G_monotone}).  Via Lemma~\ref{lemma:local_minima}
  $J_{|S|}(x)$ is monotonically decreasing and thus the maximum of
  $J_{|S|}(\widetilde{\varphi})$ must occur at the beginning of
  $A$.
\end{proof}  
\end{mdframed}

\noindent
\mypara{Lattice Concept Classes}

\begin{wrapfigure}{r}{5cm}
  \centering
  \def\svgwidth{3cm}
  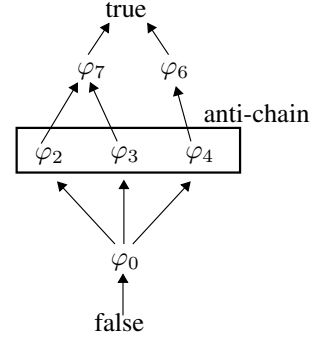
  \caption{Hasse Diagram of an example lattice $\Phi$ with an
    anti-chain annotated. Directed edges represent known subset
    relations and paths represent chains.\label{fig:Hasse}}
\end{wrapfigure}
Theorem~\ref{thm:max_f_in_chain} suggests specializing to concept
classes where determining subset relations is easy.  We propose
studying concept classes organized into a finite (bounded) lattice,
($\Phi, \pOrder$), that respects subset inclusion: $ (\varphi \pOrder
\varphi' \implies \varphi \subseteq \varphi')$.  To enforce the
bounded constraint, we assert that $\true$ and $\false$ are always
assumed to be in $\Phi$ and act as the bottom and top of the partial
order respectively.
Intuitively, this lattice structure encodes our partial knowledge of
which specifications imply other specifications. These implication
relations can be represented as a directed graph where the nodes
correspond to elements of $\Phi$ and an edge is present if the source
is known to imply the target.  Because implication is transitive, many
of the edges can be omitted without losing any information. The graph
resulting from this transitive reduction is called a Hasse diagram
~\cite{christofides1975graph} (See Fig~\ref{fig:Hasse}).
In terms of the graphical model, the Hasse diagram encodes that for
certain pairs of specifications, $\varphi, \varphi'$, we know that
$\Prob(\varphi(\xi)=1~|~\varphi'(\xi)=1, M) = 1$ or
$\Prob(\varphi(\xi)=0~|~\varphi'(\xi)=0, M) = 1$.

\noindent
\mypara{Inference on chain concept classes}
Sequences of specifications ordered by subset inclusion generalize
naturally to ascending chains.
\begin{mdframed}
  \begin{definition}[Chains and Anti-Chains]
    Given a partial order $(\Phi, \pOrder)$, an ascending \emph{chain} (or just
    chain) is a sequence of elements of $A$ ordered by $\pOrder$. The smallest element of the chain is denoted, $\downarrow(A)$. Finally,
    an \emph{anti-chain} is a set of incomparable elements. An anti-chain
    is called maximal if no element can be added to it without
    making two of its elements comparable.
  \end{definition}
\end{mdframed}
Recasting Theorem~\ref{thm:max_f_in_chain} in the parlance of chains yields:
\begin{mdframed}
  \begin{corollary}\label{cor:max_f_in_chain_partition}
    If $S \subseteq \Demos$ is a subset of demonstrations and
    $A$ is a chain in $(\Phi_S, \pOrder)$ then:
    \begin{equation}
      \max_{\varphi \in A} J_{|S|}(\widetilde{\varphi}) = J_{|S|}(\widetilde{\downarrow(A)})
    \end{equation}
  \end{corollary}  
\end{mdframed}

Observe that if the lattice, $(\Phi, \pOrder)$ is itself a chain, then
there are at most $|\Demos| + 1$ non-empty demonstration partitions,
$\Phi_S$. In fact, the non-empty partitions can be re-indexed by the
cardinality of $S$, e.g., $\Phi_S \mapsto \Phi_{|S|}$.  Further, note
that since chains are totally ordered, the smallest element of each non-empty
partition can be found by performing a binary search (indicated by find\_smallest below). These insights
are combined into Algorithm~\ref{alg:infer_chains} with a
relativized run-time analysis given in Thm~\ref{thm:runtime1}.
\begin{algorithm}[h]
  \caption{{Inference on chains} \label{alg:infer_chains}}
  \begin{algorithmic}[1]
    \Procedure{chain\_inference}{$\Demos, (A, \pOrder)$}
    \State $\Psi \gets \bigg \{(i, \text{find\_smallest}(A,
    i))~\big |~i \in \setof{0, 1, \ldots, |\Demos|} \bigg\}$\Comment{$O(|T_{\text{data}}|\Demos|\ln(|A|))$.}
    \State \Return  $\displaystyle i, \varphi^* \gets \argmax_{i, \varphi \in
      \Psi} J_i(\widetilde{\varphi})$\Comment{$O(T_{\text{rand}}|\Demos|)$}
    \EndProcedure
  \end{algorithmic}
\end{algorithm}

\begin{mdframed}[style=2, everyline=true]
\begin{theorem}\label{thm:runtime1}
  Let $T_{\text{data}}$ and $T_{\text{rand}}$ respectively
  represent the worst case execution time of computing
  $\overline{\varphi}$ and $\widetilde{\varphi}$ for $\varphi$ in chain $A$.
  Given demonstrations $\Demos$,
  Alg~\ref{alg:infer_chains} runs in time:
  \begin{equation}\label{eq:run_time1}
    O\bigg(|\Demos|\Big(T_{\text{data}}\ln(|A|) + T_{\text{rand}}\Big)\bigg)
  \end{equation}
\end{theorem}
\begin{proof}[Proof Sketch] A binary search over $|A|$ elements takes $\ln(|A|)$ time. There are $|\Demos|$ binary searches required to find the smallest element of each partition. Finally, for each smallest element, a single random
  satisfaction query is made.
\end{proof}  
\end{mdframed}

\noindent
\mypara{Lattice inference}
Of course, in general, $(\Phi, \pOrder)$ is not a chain, but a
complicated lattice. Nevertheless, observe that any path from $\false$
to $\true$ is a chain. Further, the smallest element of each partition
must either be the same specification or incomparable in $(\Phi,
\pOrder)$. That is, for each $k \in \{0, 1, \ldots |\Demos|\}$, the
set:
\begin{equation}\label{eq:demo_antichains}
  B_k \eqdef \setof{\downarrow(\Phi_S)~:~S \in {\Demos \choose k}}
\end{equation}
is a maximal anti-chain. Thus,
Corollary~\ref{cor:max_f_in_chain_partition} can be extended to:
\vspace{-2px}
\begin{mdframed}
    \begin{corollary}\label{cor:max_f_in_lattice}
    Given a lattice $(\Phi, \pOrder)$ and demonstrations $\Demos$:
    \begin{equation}
      \max_{\varphi \in \Phi} J_{N_\varphi}(\widetilde{\varphi}) =
      \max_{k \in {0, 1, \ldots, |X|}} \max_{\varphi \in B_k}J_{k}(\widetilde{\varphi})
    \end{equation}
  \end{corollary}
\end{mdframed}
\vspace{-2px}
Recalling that $N_\varphi$ increases on paths from $\false$ to $\true$,
we arrive at the following simple algorithm which takes as input the
demonstrations and the lattice $\varphi$ encoded as a directed acyclic graph
rooted at $\false$. (i) Perform a breadth first traversal (BFT) of the
lattice $(\Phi, \pOrder)$ starting at $\false$ (ii) During
the traversal, if specification $\varphi$ has a larger $N_\varphi$ than
all of its direct predecessors, then check if it is more probable
than the best specification seen so far (if so, make it the most probable
specification seen so far). (iii) At the end of the traversal, return
the most probable specification. Pseudo code is provided in
Algorithm~\ref{alg:infer_po} with a run-time analysis given in Theorem~\ref{thm:runtime2}.

\begin{algorithm}[H]
  \caption{Inference on Partial Orders\label{alg:infer_po}}
  \begin{algorithmic}[1]
    \Procedure{partialorder\_inference}{$\Demos, (\Phi, \pOrder)$}
    \State $(\varphi^*, \text{best\_info\_gain}) \gets (\false, 0)$
    \For{$\varphi \text{ in breadth\_first\_traversal}((\Phi, \pOrder))$ }
    \State $\text{parents} \gets \text{direct\_predecessors}(\varphi)$
    \If{$\exists \varphi' \in \text{parents} ~.~ N_{\varphi'} = N_\varphi$}
    \State \textbf{continue}
    \EndIf
    \State $\text{info\_gain} \gets J_{N_\varphi}(\widetilde{\varphi})$
    \If{$\text{info\_gain} > \text{best\_info\_gain}$}
    \State $(\varphi^*, \text{best\_info\_gain}) \gets (\varphi, \text{info\_gain})$
    \EndIf
    \EndFor
    \State \Return $\varphi^*$ \EndProcedure
  \end{algorithmic}
\end{algorithm}
\begin{mdframed}
  \begin{theorem}\label{thm:runtime2}
    Let $(\Phi, \pOrder)$ be a bounded partial order encoded as a
    Directed Acyclic Graph (DAG), $G = (V, E)$, with vertices $V$ and
    edges $E$. Further, let $B$ denote the largest anti-chain in
    $\Phi$.  If $T_{\text{data}}$ and $T_{\text{rand}}$ respectively
    represent the worst case execution time of computing
    $\overline{\varphi}$ and $\widetilde{\varphi}$, then
    for demonstrations $\Demos$, Alg~\ref{alg:infer_po} runs in time:
    \begin{equation}\label{eq:run_time2}
      O \big (E + T_{\text{data}}\cdot V + T_{\text{rand}}\cdot|B||\Demos|\big )
    \end{equation}
  \end{theorem}
  \begin{proof}[Proof sketch]
    BFT takes $O(V + E)$. Further, for each node,
    $\overline{\varphi}$ is computed ($O(T_\text{data}\cdot V)$). Finally, for each node in each of the
    candidate anti-chains $B_k$, $\widetilde{\varphi}$ is computed.
    Since $|B|$ is the size of the largest anti-chain, this query
    happens no more than $|B||\Demos|$ times.
  \end{proof}
\end{mdframed}


%% file: imgs/thm.pgf
\begingroup%
\makeatletter%
\begin{pgfpicture}%
\pgfpathrectangle{\pgfpointorigin}{\pgfqpoint{6.000000in}{4.000000in}}%
\pgfusepath{use as bounding box, clip}%
\begin{pgfscope}%
\pgfsetbuttcap%
\pgfsetmiterjoin%
\definecolor{currentfill}{rgb}{1.000000,1.000000,1.000000}%
\pgfsetfillcolor{currentfill}%
\pgfsetlinewidth{0.000000pt}%
\definecolor{currentstroke}{rgb}{1.000000,1.000000,1.000000}%
\pgfsetstrokecolor{currentstroke}%
\pgfsetdash{}{0pt}%
\pgfpathmoveto{\pgfqpoint{0.000000in}{0.000000in}}%
\pgfpathlineto{\pgfqpoint{6.000000in}{0.000000in}}%
\pgfpathlineto{\pgfqpoint{6.000000in}{4.000000in}}%
\pgfpathlineto{\pgfqpoint{0.000000in}{4.000000in}}%
\pgfpathclose%
\pgfusepath{fill}%
\end{pgfscope}%
\begin{pgfscope}%
\pgfsetbuttcap%
\pgfsetmiterjoin%
\definecolor{currentfill}{rgb}{0.917647,0.917647,0.949020}%
\pgfsetfillcolor{currentfill}%
\pgfsetlinewidth{0.000000pt}%
\definecolor{currentstroke}{rgb}{0.000000,0.000000,0.000000}%
\pgfsetstrokecolor{currentstroke}%
\pgfsetstrokeopacity{0.000000}%
\pgfsetdash{}{0pt}%
\pgfpathmoveto{\pgfqpoint{0.893953in}{1.021278in}}%
\pgfpathlineto{\pgfqpoint{5.610667in}{1.021278in}}%
\pgfpathlineto{\pgfqpoint{5.610667in}{3.402333in}}%
\pgfpathlineto{\pgfqpoint{0.893953in}{3.402333in}}%
\pgfpathclose%
\pgfusepath{fill}%
\end{pgfscope}%
\begin{pgfscope}%
\pgfpathrectangle{\pgfqpoint{0.893953in}{1.021278in}}{\pgfqpoint{4.716714in}{2.381056in}}%
\pgfusepath{clip}%
\pgfsetroundcap%
\pgfsetroundjoin%
\pgfsetlinewidth{1.003750pt}%
\definecolor{currentstroke}{rgb}{1.000000,1.000000,1.000000}%
\pgfsetstrokecolor{currentstroke}%
\pgfsetdash{}{0pt}%
\pgfpathmoveto{\pgfqpoint{1.107915in}{1.021278in}}%
\pgfpathlineto{\pgfqpoint{1.107915in}{3.402333in}}%
\pgfusepath{stroke}%
\end{pgfscope}%
\begin{pgfscope}%
\definecolor{textcolor}{rgb}{0.150000,0.150000,0.150000}%
\pgfsetstrokecolor{textcolor}%
\pgfsetfillcolor{textcolor}%
\pgftext[x=1.107915in,y=0.889333in,,top]{\color{textcolor}\sffamily\fontsize{18.700000}{22.440000}\selectfont \(\displaystyle 0.0\)}%
\end{pgfscope}%
\begin{pgfscope}%
\pgfpathrectangle{\pgfqpoint{0.893953in}{1.021278in}}{\pgfqpoint{4.716714in}{2.381056in}}%
\pgfusepath{clip}%
\pgfsetroundcap%
\pgfsetroundjoin%
\pgfsetlinewidth{1.003750pt}%
\definecolor{currentstroke}{rgb}{1.000000,1.000000,1.000000}%
\pgfsetstrokecolor{currentstroke}%
\pgfsetdash{}{0pt}%
\pgfpathmoveto{\pgfqpoint{1.974250in}{1.021278in}}%
\pgfpathlineto{\pgfqpoint{1.974250in}{3.402333in}}%
\pgfusepath{stroke}%
\end{pgfscope}%
\begin{pgfscope}%
\definecolor{textcolor}{rgb}{0.150000,0.150000,0.150000}%
\pgfsetstrokecolor{textcolor}%
\pgfsetfillcolor{textcolor}%
\pgftext[x=1.974250in,y=0.889333in,,top]{\color{textcolor}\sffamily\fontsize{18.700000}{22.440000}\selectfont \(\displaystyle 0.2\)}%
\end{pgfscope}%
\begin{pgfscope}%
\pgfpathrectangle{\pgfqpoint{0.893953in}{1.021278in}}{\pgfqpoint{4.716714in}{2.381056in}}%
\pgfusepath{clip}%
\pgfsetroundcap%
\pgfsetroundjoin%
\pgfsetlinewidth{1.003750pt}%
\definecolor{currentstroke}{rgb}{1.000000,1.000000,1.000000}%
\pgfsetstrokecolor{currentstroke}%
\pgfsetdash{}{0pt}%
\pgfpathmoveto{\pgfqpoint{2.840584in}{1.021278in}}%
\pgfpathlineto{\pgfqpoint{2.840584in}{3.402333in}}%
\pgfusepath{stroke}%
\end{pgfscope}%
\begin{pgfscope}%
\definecolor{textcolor}{rgb}{0.150000,0.150000,0.150000}%
\pgfsetstrokecolor{textcolor}%
\pgfsetfillcolor{textcolor}%
\pgftext[x=2.840584in,y=0.889333in,,top]{\color{textcolor}\sffamily\fontsize{18.700000}{22.440000}\selectfont \(\displaystyle 0.4\)}%
\end{pgfscope}%
\begin{pgfscope}%
\pgfpathrectangle{\pgfqpoint{0.893953in}{1.021278in}}{\pgfqpoint{4.716714in}{2.381056in}}%
\pgfusepath{clip}%
\pgfsetroundcap%
\pgfsetroundjoin%
\pgfsetlinewidth{1.003750pt}%
\definecolor{currentstroke}{rgb}{1.000000,1.000000,1.000000}%
\pgfsetstrokecolor{currentstroke}%
\pgfsetdash{}{0pt}%
\pgfpathmoveto{\pgfqpoint{3.706919in}{1.021278in}}%
\pgfpathlineto{\pgfqpoint{3.706919in}{3.402333in}}%
\pgfusepath{stroke}%
\end{pgfscope}%
\begin{pgfscope}%
\definecolor{textcolor}{rgb}{0.150000,0.150000,0.150000}%
\pgfsetstrokecolor{textcolor}%
\pgfsetfillcolor{textcolor}%
\pgftext[x=3.706919in,y=0.889333in,,top]{\color{textcolor}\sffamily\fontsize{18.700000}{22.440000}\selectfont \(\displaystyle 0.6\)}%
\end{pgfscope}%
\begin{pgfscope}%
\pgfpathrectangle{\pgfqpoint{0.893953in}{1.021278in}}{\pgfqpoint{4.716714in}{2.381056in}}%
\pgfusepath{clip}%
\pgfsetroundcap%
\pgfsetroundjoin%
\pgfsetlinewidth{1.003750pt}%
\definecolor{currentstroke}{rgb}{1.000000,1.000000,1.000000}%
\pgfsetstrokecolor{currentstroke}%
\pgfsetdash{}{0pt}%
\pgfpathmoveto{\pgfqpoint{4.573253in}{1.021278in}}%
\pgfpathlineto{\pgfqpoint{4.573253in}{3.402333in}}%
\pgfusepath{stroke}%
\end{pgfscope}%
\begin{pgfscope}%
\definecolor{textcolor}{rgb}{0.150000,0.150000,0.150000}%
\pgfsetstrokecolor{textcolor}%
\pgfsetfillcolor{textcolor}%
\pgftext[x=4.573253in,y=0.889333in,,top]{\color{textcolor}\sffamily\fontsize{18.700000}{22.440000}\selectfont \(\displaystyle 0.8\)}%
\end{pgfscope}%
\begin{pgfscope}%
\pgfpathrectangle{\pgfqpoint{0.893953in}{1.021278in}}{\pgfqpoint{4.716714in}{2.381056in}}%
\pgfusepath{clip}%
\pgfsetroundcap%
\pgfsetroundjoin%
\pgfsetlinewidth{1.003750pt}%
\definecolor{currentstroke}{rgb}{1.000000,1.000000,1.000000}%
\pgfsetstrokecolor{currentstroke}%
\pgfsetdash{}{0pt}%
\pgfpathmoveto{\pgfqpoint{5.439587in}{1.021278in}}%
\pgfpathlineto{\pgfqpoint{5.439587in}{3.402333in}}%
\pgfusepath{stroke}%
\end{pgfscope}%
\begin{pgfscope}%
\definecolor{textcolor}{rgb}{0.150000,0.150000,0.150000}%
\pgfsetstrokecolor{textcolor}%
\pgfsetfillcolor{textcolor}%
\pgftext[x=5.439587in,y=0.889333in,,top]{\color{textcolor}\sffamily\fontsize{18.700000}{22.440000}\selectfont \(\displaystyle 1.0\)}%
\end{pgfscope}%
\begin{pgfscope}%
\definecolor{textcolor}{rgb}{0.150000,0.150000,0.150000}%
\pgfsetstrokecolor{textcolor}%
\pgfsetfillcolor{textcolor}%
\pgftext[x=3.252310in,y=0.582426in,,top]{\color{textcolor}\sffamily\fontsize{20.400000}{24.480000}\selectfont \(\displaystyle x\)}%
\end{pgfscope}%
\begin{pgfscope}%
\pgfpathrectangle{\pgfqpoint{0.893953in}{1.021278in}}{\pgfqpoint{4.716714in}{2.381056in}}%
\pgfusepath{clip}%
\pgfsetroundcap%
\pgfsetroundjoin%
\pgfsetlinewidth{1.003750pt}%
\definecolor{currentstroke}{rgb}{1.000000,1.000000,1.000000}%
\pgfsetstrokecolor{currentstroke}%
\pgfsetdash{}{0pt}%
\pgfpathmoveto{\pgfqpoint{0.893953in}{1.158249in}}%
\pgfpathlineto{\pgfqpoint{5.610667in}{1.158249in}}%
\pgfusepath{stroke}%
\end{pgfscope}%
\begin{pgfscope}%
\definecolor{textcolor}{rgb}{0.150000,0.150000,0.150000}%
\pgfsetstrokecolor{textcolor}%
\pgfsetfillcolor{textcolor}%
\pgftext[x=0.651940in,y=1.059585in,left,base]{\color{textcolor}\sffamily\fontsize{18.700000}{22.440000}\selectfont \(\displaystyle 0\)}%
\end{pgfscope}%
\begin{pgfscope}%
\pgfpathrectangle{\pgfqpoint{0.893953in}{1.021278in}}{\pgfqpoint{4.716714in}{2.381056in}}%
\pgfusepath{clip}%
\pgfsetroundcap%
\pgfsetroundjoin%
\pgfsetlinewidth{1.003750pt}%
\definecolor{currentstroke}{rgb}{1.000000,1.000000,1.000000}%
\pgfsetstrokecolor{currentstroke}%
\pgfsetdash{}{0pt}%
\pgfpathmoveto{\pgfqpoint{0.893953in}{1.934395in}}%
\pgfpathlineto{\pgfqpoint{5.610667in}{1.934395in}}%
\pgfusepath{stroke}%
\end{pgfscope}%
\begin{pgfscope}%
\definecolor{textcolor}{rgb}{0.150000,0.150000,0.150000}%
\pgfsetstrokecolor{textcolor}%
\pgfsetfillcolor{textcolor}%
\pgftext[x=0.651940in,y=1.835731in,left,base]{\color{textcolor}\sffamily\fontsize{18.700000}{22.440000}\selectfont \(\displaystyle 2\)}%
\end{pgfscope}%
\begin{pgfscope}%
\pgfpathrectangle{\pgfqpoint{0.893953in}{1.021278in}}{\pgfqpoint{4.716714in}{2.381056in}}%
\pgfusepath{clip}%
\pgfsetroundcap%
\pgfsetroundjoin%
\pgfsetlinewidth{1.003750pt}%
\definecolor{currentstroke}{rgb}{1.000000,1.000000,1.000000}%
\pgfsetstrokecolor{currentstroke}%
\pgfsetdash{}{0pt}%
\pgfpathmoveto{\pgfqpoint{0.893953in}{2.710540in}}%
\pgfpathlineto{\pgfqpoint{5.610667in}{2.710540in}}%
\pgfusepath{stroke}%
\end{pgfscope}%
\begin{pgfscope}%
\definecolor{textcolor}{rgb}{0.150000,0.150000,0.150000}%
\pgfsetstrokecolor{textcolor}%
\pgfsetfillcolor{textcolor}%
\pgftext[x=0.651940in,y=2.611876in,left,base]{\color{textcolor}\sffamily\fontsize{18.700000}{22.440000}\selectfont \(\displaystyle 4\)}%
\end{pgfscope}%
\begin{pgfscope}%
\definecolor{textcolor}{rgb}{0.150000,0.150000,0.150000}%
\pgfsetstrokecolor{textcolor}%
\pgfsetfillcolor{textcolor}%
\pgftext[x=0.596384in,y=2.211806in,,bottom,rotate=90.000000]{\color{textcolor}\sffamily\fontsize{20.400000}{24.480000}\selectfont \(\displaystyle J_{6}(x)\)}%
\end{pgfscope}%
\begin{pgfscope}%
\pgfpathrectangle{\pgfqpoint{0.893953in}{1.021278in}}{\pgfqpoint{4.716714in}{2.381056in}}%
\pgfusepath{clip}%
\pgfsetbuttcap%
\pgfsetroundjoin%
\definecolor{currentfill}{rgb}{0.768627,0.305882,0.321569}%
\pgfsetfillcolor{currentfill}%
\pgfsetlinewidth{1.003750pt}%
\definecolor{currentstroke}{rgb}{0.768627,0.305882,0.321569}%
\pgfsetstrokecolor{currentstroke}%
\pgfsetdash{}{0pt}%
\pgfpathmoveto{\pgfqpoint{1.151232in}{2.062296in}}%
\pgfpathcurveto{\pgfqpoint{1.162282in}{2.062296in}}{\pgfqpoint{1.172881in}{2.066686in}}{\pgfqpoint{1.180695in}{2.074500in}}%
\pgfpathcurveto{\pgfqpoint{1.188509in}{2.082314in}}{\pgfqpoint{1.192899in}{2.092913in}}{\pgfqpoint{1.192899in}{2.103963in}}%
\pgfpathcurveto{\pgfqpoint{1.192899in}{2.115013in}}{\pgfqpoint{1.188509in}{2.125612in}}{\pgfqpoint{1.180695in}{2.133426in}}%
\pgfpathcurveto{\pgfqpoint{1.172881in}{2.141239in}}{\pgfqpoint{1.162282in}{2.145630in}}{\pgfqpoint{1.151232in}{2.145630in}}%
\pgfpathcurveto{\pgfqpoint{1.140182in}{2.145630in}}{\pgfqpoint{1.129583in}{2.141239in}}{\pgfqpoint{1.121769in}{2.133426in}}%
\pgfpathcurveto{\pgfqpoint{1.113956in}{2.125612in}}{\pgfqpoint{1.109566in}{2.115013in}}{\pgfqpoint{1.109566in}{2.103963in}}%
\pgfpathcurveto{\pgfqpoint{1.109566in}{2.092913in}}{\pgfqpoint{1.113956in}{2.082314in}}{\pgfqpoint{1.121769in}{2.074500in}}%
\pgfpathcurveto{\pgfqpoint{1.129583in}{2.066686in}}{\pgfqpoint{1.140182in}{2.062296in}}{\pgfqpoint{1.151232in}{2.062296in}}%
\pgfpathclose%
\pgfusepath{stroke,fill}%
\end{pgfscope}%
\begin{pgfscope}%
\pgfpathrectangle{\pgfqpoint{0.893953in}{1.021278in}}{\pgfqpoint{4.716714in}{2.381056in}}%
\pgfusepath{clip}%
\pgfsetbuttcap%
\pgfsetroundjoin%
\definecolor{currentfill}{rgb}{0.768627,0.305882,0.321569}%
\pgfsetfillcolor{currentfill}%
\pgfsetlinewidth{1.003750pt}%
\definecolor{currentstroke}{rgb}{0.768627,0.305882,0.321569}%
\pgfsetstrokecolor{currentstroke}%
\pgfsetdash{}{0pt}%
\pgfpathmoveto{\pgfqpoint{1.974250in}{1.314820in}}%
\pgfpathcurveto{\pgfqpoint{1.985300in}{1.314820in}}{\pgfqpoint{1.995899in}{1.319210in}}{\pgfqpoint{2.003713in}{1.327024in}}%
\pgfpathcurveto{\pgfqpoint{2.011526in}{1.334837in}}{\pgfqpoint{2.015916in}{1.345436in}}{\pgfqpoint{2.015916in}{1.356486in}}%
\pgfpathcurveto{\pgfqpoint{2.015916in}{1.367537in}}{\pgfqpoint{2.011526in}{1.378136in}}{\pgfqpoint{2.003713in}{1.385949in}}%
\pgfpathcurveto{\pgfqpoint{1.995899in}{1.393763in}}{\pgfqpoint{1.985300in}{1.398153in}}{\pgfqpoint{1.974250in}{1.398153in}}%
\pgfpathcurveto{\pgfqpoint{1.963200in}{1.398153in}}{\pgfqpoint{1.952601in}{1.393763in}}{\pgfqpoint{1.944787in}{1.385949in}}%
\pgfpathcurveto{\pgfqpoint{1.936973in}{1.378136in}}{\pgfqpoint{1.932583in}{1.367537in}}{\pgfqpoint{1.932583in}{1.356486in}}%
\pgfpathcurveto{\pgfqpoint{1.932583in}{1.345436in}}{\pgfqpoint{1.936973in}{1.334837in}}{\pgfqpoint{1.944787in}{1.327024in}}%
\pgfpathcurveto{\pgfqpoint{1.952601in}{1.319210in}}{\pgfqpoint{1.963200in}{1.314820in}}{\pgfqpoint{1.974250in}{1.314820in}}%
\pgfpathclose%
\pgfusepath{stroke,fill}%
\end{pgfscope}%
\begin{pgfscope}%
\pgfpathrectangle{\pgfqpoint{0.893953in}{1.021278in}}{\pgfqpoint{4.716714in}{2.381056in}}%
\pgfusepath{clip}%
\pgfsetbuttcap%
\pgfsetroundjoin%
\definecolor{currentfill}{rgb}{0.768627,0.305882,0.321569}%
\pgfsetfillcolor{currentfill}%
\pgfsetlinewidth{1.003750pt}%
\definecolor{currentstroke}{rgb}{0.768627,0.305882,0.321569}%
\pgfsetstrokecolor{currentstroke}%
\pgfsetdash{}{0pt}%
\pgfpathmoveto{\pgfqpoint{2.407417in}{1.227193in}}%
\pgfpathcurveto{\pgfqpoint{2.418467in}{1.227193in}}{\pgfqpoint{2.429066in}{1.231583in}}{\pgfqpoint{2.436880in}{1.239397in}}%
\pgfpathcurveto{\pgfqpoint{2.444693in}{1.247211in}}{\pgfqpoint{2.449084in}{1.257810in}}{\pgfqpoint{2.449084in}{1.268860in}}%
\pgfpathcurveto{\pgfqpoint{2.449084in}{1.279910in}}{\pgfqpoint{2.444693in}{1.290509in}}{\pgfqpoint{2.436880in}{1.298323in}}%
\pgfpathcurveto{\pgfqpoint{2.429066in}{1.306136in}}{\pgfqpoint{2.418467in}{1.310526in}}{\pgfqpoint{2.407417in}{1.310526in}}%
\pgfpathcurveto{\pgfqpoint{2.396367in}{1.310526in}}{\pgfqpoint{2.385768in}{1.306136in}}{\pgfqpoint{2.377954in}{1.298323in}}%
\pgfpathcurveto{\pgfqpoint{2.370141in}{1.290509in}}{\pgfqpoint{2.365750in}{1.279910in}}{\pgfqpoint{2.365750in}{1.268860in}}%
\pgfpathcurveto{\pgfqpoint{2.365750in}{1.257810in}}{\pgfqpoint{2.370141in}{1.247211in}}{\pgfqpoint{2.377954in}{1.239397in}}%
\pgfpathcurveto{\pgfqpoint{2.385768in}{1.231583in}}{\pgfqpoint{2.396367in}{1.227193in}}{\pgfqpoint{2.407417in}{1.227193in}}%
\pgfpathclose%
\pgfusepath{stroke,fill}%
\end{pgfscope}%
\begin{pgfscope}%
\pgfpathrectangle{\pgfqpoint{0.893953in}{1.021278in}}{\pgfqpoint{4.716714in}{2.381056in}}%
\pgfusepath{clip}%
\pgfsetbuttcap%
\pgfsetroundjoin%
\definecolor{currentfill}{rgb}{0.768627,0.305882,0.321569}%
\pgfsetfillcolor{currentfill}%
\pgfsetlinewidth{1.003750pt}%
\definecolor{currentstroke}{rgb}{0.768627,0.305882,0.321569}%
\pgfsetstrokecolor{currentstroke}%
\pgfsetdash{}{0pt}%
\pgfpathmoveto{\pgfqpoint{3.706919in}{1.120256in}}%
\pgfpathcurveto{\pgfqpoint{3.717969in}{1.120256in}}{\pgfqpoint{3.728568in}{1.124646in}}{\pgfqpoint{3.736381in}{1.132460in}}%
\pgfpathcurveto{\pgfqpoint{3.744195in}{1.140273in}}{\pgfqpoint{3.748585in}{1.150872in}}{\pgfqpoint{3.748585in}{1.161923in}}%
\pgfpathcurveto{\pgfqpoint{3.748585in}{1.172973in}}{\pgfqpoint{3.744195in}{1.183572in}}{\pgfqpoint{3.736381in}{1.191385in}}%
\pgfpathcurveto{\pgfqpoint{3.728568in}{1.199199in}}{\pgfqpoint{3.717969in}{1.203589in}}{\pgfqpoint{3.706919in}{1.203589in}}%
\pgfpathcurveto{\pgfqpoint{3.695868in}{1.203589in}}{\pgfqpoint{3.685269in}{1.199199in}}{\pgfqpoint{3.677456in}{1.191385in}}%
\pgfpathcurveto{\pgfqpoint{3.669642in}{1.183572in}}{\pgfqpoint{3.665252in}{1.172973in}}{\pgfqpoint{3.665252in}{1.161923in}}%
\pgfpathcurveto{\pgfqpoint{3.665252in}{1.150872in}}{\pgfqpoint{3.669642in}{1.140273in}}{\pgfqpoint{3.677456in}{1.132460in}}%
\pgfpathcurveto{\pgfqpoint{3.685269in}{1.124646in}}{\pgfqpoint{3.695868in}{1.120256in}}{\pgfqpoint{3.706919in}{1.120256in}}%
\pgfpathclose%
\pgfusepath{stroke,fill}%
\end{pgfscope}%
\begin{pgfscope}%
\pgfpathrectangle{\pgfqpoint{0.893953in}{1.021278in}}{\pgfqpoint{4.716714in}{2.381056in}}%
\pgfusepath{clip}%
\pgfsetroundcap%
\pgfsetroundjoin%
\pgfsetlinewidth{1.505625pt}%
\definecolor{currentstroke}{rgb}{0.298039,0.447059,0.690196}%
\pgfsetstrokecolor{currentstroke}%
\pgfsetdash{}{0pt}%
\pgfpathmoveto{\pgfqpoint{1.108349in}{3.294104in}}%
\pgfpathlineto{\pgfqpoint{1.151661in}{2.101427in}}%
\pgfpathlineto{\pgfqpoint{1.194973in}{1.924697in}}%
\pgfpathlineto{\pgfqpoint{1.238286in}{1.821553in}}%
\pgfpathlineto{\pgfqpoint{1.281598in}{1.748681in}}%
\pgfpathlineto{\pgfqpoint{1.324910in}{1.692434in}}%
\pgfpathlineto{\pgfqpoint{1.368223in}{1.646720in}}%
\pgfpathlineto{\pgfqpoint{1.411535in}{1.608283in}}%
\pgfpathlineto{\pgfqpoint{1.454847in}{1.575181in}}%
\pgfpathlineto{\pgfqpoint{1.498160in}{1.546159in}}%
\pgfpathlineto{\pgfqpoint{1.541472in}{1.520359in}}%
\pgfpathlineto{\pgfqpoint{1.584784in}{1.497169in}}%
\pgfpathlineto{\pgfqpoint{1.628097in}{1.476139in}}%
\pgfpathlineto{\pgfqpoint{1.671409in}{1.456926in}}%
\pgfpathlineto{\pgfqpoint{1.714721in}{1.439263in}}%
\pgfpathlineto{\pgfqpoint{1.758034in}{1.422939in}}%
\pgfpathlineto{\pgfqpoint{1.801346in}{1.407783in}}%
\pgfpathlineto{\pgfqpoint{1.844658in}{1.393657in}}%
\pgfpathlineto{\pgfqpoint{1.887971in}{1.380446in}}%
\pgfpathlineto{\pgfqpoint{1.931283in}{1.368053in}}%
\pgfpathlineto{\pgfqpoint{1.974595in}{1.356396in}}%
\pgfpathlineto{\pgfqpoint{2.017908in}{1.345407in}}%
\pgfpathlineto{\pgfqpoint{2.061220in}{1.335025in}}%
\pgfpathlineto{\pgfqpoint{2.104533in}{1.325199in}}%
\pgfpathlineto{\pgfqpoint{2.147845in}{1.315884in}}%
\pgfpathlineto{\pgfqpoint{2.191157in}{1.307040in}}%
\pgfpathlineto{\pgfqpoint{2.234470in}{1.298633in}}%
\pgfpathlineto{\pgfqpoint{2.277782in}{1.290633in}}%
\pgfpathlineto{\pgfqpoint{2.321094in}{1.283012in}}%
\pgfpathlineto{\pgfqpoint{2.364407in}{1.275746in}}%
\pgfpathlineto{\pgfqpoint{2.407719in}{1.268813in}}%
\pgfpathlineto{\pgfqpoint{2.451031in}{1.262193in}}%
\pgfpathlineto{\pgfqpoint{2.494344in}{1.255871in}}%
\pgfpathlineto{\pgfqpoint{2.537656in}{1.249828in}}%
\pgfpathlineto{\pgfqpoint{2.580968in}{1.244053in}}%
\pgfpathlineto{\pgfqpoint{2.624281in}{1.238530in}}%
\pgfpathlineto{\pgfqpoint{2.667593in}{1.233250in}}%
\pgfpathlineto{\pgfqpoint{2.710905in}{1.228200in}}%
\pgfpathlineto{\pgfqpoint{2.754218in}{1.223372in}}%
\pgfpathlineto{\pgfqpoint{2.797530in}{1.218757in}}%
\pgfpathlineto{\pgfqpoint{2.840842in}{1.214347in}}%
\pgfpathlineto{\pgfqpoint{2.884155in}{1.210134in}}%
\pgfpathlineto{\pgfqpoint{2.927467in}{1.206113in}}%
\pgfpathlineto{\pgfqpoint{2.970779in}{1.202276in}}%
\pgfpathlineto{\pgfqpoint{3.014092in}{1.198619in}}%
\pgfpathlineto{\pgfqpoint{3.057404in}{1.195137in}}%
\pgfpathlineto{\pgfqpoint{3.100716in}{1.191826in}}%
\pgfpathlineto{\pgfqpoint{3.144029in}{1.188681in}}%
\pgfpathlineto{\pgfqpoint{3.187341in}{1.185699in}}%
\pgfpathlineto{\pgfqpoint{3.230653in}{1.182878in}}%
\pgfpathlineto{\pgfqpoint{3.273966in}{1.180214in}}%
\pgfpathlineto{\pgfqpoint{3.317278in}{1.177705in}}%
\pgfpathlineto{\pgfqpoint{3.360590in}{1.175349in}}%
\pgfpathlineto{\pgfqpoint{3.403903in}{1.173146in}}%
\pgfpathlineto{\pgfqpoint{3.447215in}{1.171093in}}%
\pgfpathlineto{\pgfqpoint{3.490527in}{1.169189in}}%
\pgfpathlineto{\pgfqpoint{3.533840in}{1.167436in}}%
\pgfpathlineto{\pgfqpoint{3.577152in}{1.165831in}}%
\pgfpathlineto{\pgfqpoint{3.620465in}{1.164376in}}%
\pgfpathlineto{\pgfqpoint{3.663777in}{1.163072in}}%
\pgfpathlineto{\pgfqpoint{3.707089in}{1.161918in}}%
\pgfpathlineto{\pgfqpoint{3.750402in}{1.160918in}}%
\pgfpathlineto{\pgfqpoint{3.793714in}{1.160072in}}%
\pgfpathlineto{\pgfqpoint{3.837026in}{1.159383in}}%
\pgfpathlineto{\pgfqpoint{3.880339in}{1.158853in}}%
\pgfpathlineto{\pgfqpoint{3.923651in}{1.158487in}}%
\pgfpathlineto{\pgfqpoint{3.966963in}{1.158287in}}%
\pgfpathlineto{\pgfqpoint{4.010276in}{1.158249in}}%
\pgfpathlineto{\pgfqpoint{4.053588in}{1.158249in}}%
\pgfpathlineto{\pgfqpoint{4.096900in}{1.158249in}}%
\pgfpathlineto{\pgfqpoint{4.140213in}{1.158249in}}%
\pgfpathlineto{\pgfqpoint{4.183525in}{1.158249in}}%
\pgfpathlineto{\pgfqpoint{4.226837in}{1.158249in}}%
\pgfpathlineto{\pgfqpoint{4.270150in}{1.158249in}}%
\pgfpathlineto{\pgfqpoint{4.313462in}{1.158249in}}%
\pgfpathlineto{\pgfqpoint{4.356774in}{1.158249in}}%
\pgfpathlineto{\pgfqpoint{4.400087in}{1.158249in}}%
\pgfpathlineto{\pgfqpoint{4.443399in}{1.158249in}}%
\pgfpathlineto{\pgfqpoint{4.486711in}{1.158249in}}%
\pgfpathlineto{\pgfqpoint{4.530024in}{1.158249in}}%
\pgfpathlineto{\pgfqpoint{4.573336in}{1.158249in}}%
\pgfpathlineto{\pgfqpoint{4.616648in}{1.158249in}}%
\pgfpathlineto{\pgfqpoint{4.659961in}{1.158249in}}%
\pgfpathlineto{\pgfqpoint{4.703273in}{1.158249in}}%
\pgfpathlineto{\pgfqpoint{4.746585in}{1.158249in}}%
\pgfpathlineto{\pgfqpoint{4.789898in}{1.158249in}}%
\pgfpathlineto{\pgfqpoint{4.833210in}{1.158249in}}%
\pgfpathlineto{\pgfqpoint{4.876522in}{1.158249in}}%
\pgfpathlineto{\pgfqpoint{4.919835in}{1.158249in}}%
\pgfpathlineto{\pgfqpoint{4.963147in}{1.158249in}}%
\pgfpathlineto{\pgfqpoint{5.006459in}{1.158249in}}%
\pgfpathlineto{\pgfqpoint{5.049772in}{1.158249in}}%
\pgfpathlineto{\pgfqpoint{5.093084in}{1.158249in}}%
\pgfpathlineto{\pgfqpoint{5.136397in}{1.158249in}}%
\pgfpathlineto{\pgfqpoint{5.179709in}{1.158249in}}%
\pgfpathlineto{\pgfqpoint{5.223021in}{1.158249in}}%
\pgfpathlineto{\pgfqpoint{5.266334in}{1.158249in}}%
\pgfpathlineto{\pgfqpoint{5.309646in}{1.158249in}}%
\pgfpathlineto{\pgfqpoint{5.352958in}{1.158249in}}%
\pgfpathlineto{\pgfqpoint{5.396271in}{1.158249in}}%
\pgfusepath{stroke}%
\end{pgfscope}%
\begin{pgfscope}%
\pgfsetrectcap%
\pgfsetmiterjoin%
\pgfsetlinewidth{1.254687pt}%
\definecolor{currentstroke}{rgb}{1.000000,1.000000,1.000000}%
\pgfsetstrokecolor{currentstroke}%
\pgfsetdash{}{0pt}%
\pgfpathmoveto{\pgfqpoint{0.893953in}{1.021278in}}%
\pgfpathlineto{\pgfqpoint{0.893953in}{3.402333in}}%
\pgfusepath{stroke}%
\end{pgfscope}%
\begin{pgfscope}%
\pgfsetrectcap%
\pgfsetmiterjoin%
\pgfsetlinewidth{1.254687pt}%
\definecolor{currentstroke}{rgb}{1.000000,1.000000,1.000000}%
\pgfsetstrokecolor{currentstroke}%
\pgfsetdash{}{0pt}%
\pgfpathmoveto{\pgfqpoint{5.610667in}{1.021278in}}%
\pgfpathlineto{\pgfqpoint{5.610667in}{3.402333in}}%
\pgfusepath{stroke}%
\end{pgfscope}%
\begin{pgfscope}%
\pgfsetrectcap%
\pgfsetmiterjoin%
\pgfsetlinewidth{1.254687pt}%
\definecolor{currentstroke}{rgb}{1.000000,1.000000,1.000000}%
\pgfsetstrokecolor{currentstroke}%
\pgfsetdash{}{0pt}%
\pgfpathmoveto{\pgfqpoint{0.893953in}{1.021278in}}%
\pgfpathlineto{\pgfqpoint{5.610667in}{1.021278in}}%
\pgfusepath{stroke}%
\end{pgfscope}%
\begin{pgfscope}%
\pgfsetrectcap%
\pgfsetmiterjoin%
\pgfsetlinewidth{1.254687pt}%
\definecolor{currentstroke}{rgb}{1.000000,1.000000,1.000000}%
\pgfsetstrokecolor{currentstroke}%
\pgfsetdash{}{0pt}%
\pgfpathmoveto{\pgfqpoint{0.893953in}{3.402333in}}%
\pgfpathlineto{\pgfqpoint{5.610667in}{3.402333in}}%
\pgfusepath{stroke}%
\end{pgfscope}%
\begin{pgfscope}%
\definecolor{textcolor}{rgb}{0.150000,0.150000,0.150000}%
\pgfsetstrokecolor{textcolor}%
\pgfsetfillcolor{textcolor}%
\pgftext[x=1.194549in,y=2.103963in,left,base]{\color{textcolor}\sffamily\fontsize{20.400000}{24.480000}\selectfont \(\displaystyle \varphi_1\)}%
\end{pgfscope}%
\begin{pgfscope}%
\definecolor{textcolor}{rgb}{0.150000,0.150000,0.150000}%
\pgfsetstrokecolor{textcolor}%
\pgfsetfillcolor{textcolor}%
\pgftext[x=1.974250in,y=1.434101in,left,base]{\color{textcolor}\sffamily\fontsize{20.400000}{24.480000}\selectfont \(\displaystyle \varphi_2\)}%
\end{pgfscope}%
\begin{pgfscope}%
\definecolor{textcolor}{rgb}{0.150000,0.150000,0.150000}%
\pgfsetstrokecolor{textcolor}%
\pgfsetfillcolor{textcolor}%
\pgftext[x=2.407417in,y=1.346474in,left,base]{\color{textcolor}\sffamily\fontsize{20.400000}{24.480000}\selectfont \(\displaystyle \varphi_3\)}%
\end{pgfscope}%
\begin{pgfscope}%
\definecolor{textcolor}{rgb}{0.150000,0.150000,0.150000}%
\pgfsetstrokecolor{textcolor}%
\pgfsetfillcolor{textcolor}%
\pgftext[x=3.706919in,y=1.239537in,left,base]{\color{textcolor}\sffamily\fontsize{20.400000}{24.480000}\selectfont \(\displaystyle \varphi_4\)}%
\end{pgfscope}%
\begin{pgfscope}%
\definecolor{textcolor}{rgb}{0.150000,0.150000,0.150000}%
\pgfsetstrokecolor{textcolor}%
\pgfsetfillcolor{textcolor}%
\pgftext[x=3.252310in,y=3.485667in,,base]{\color{textcolor}\sffamily\fontsize{20.400000}{24.480000}\selectfont Illustration of Thm~\ref{thm:max_f_in_chain}}%
\end{pgfscope}%
\end{pgfpicture}%
\makeatother%
\endgroup%

%% file: imgs/lattice.pdf_tex
\begingroup%
  \makeatletter%
  \providecommand\color[2][]{%
    \errmessage{(Inkscape) Color is used for the text in Inkscape, but the package 'color.sty' is not loaded}%
    \renewcommand\color[2][]{}%
  }%
  \providecommand\transparent[1]{%
    \errmessage{(Inkscape) Transparency is used (non-zero) for the text in Inkscape, but the package 'transparent.sty' is not loaded}%
    \renewcommand\transparent[1]{}%
  }%
  \providecommand\rotatebox[2]{#2}%
  \newcommand*\fsize{\dimexpr\f@size pt\relax}%
  \newcommand*\lineheight[1]{\fontsize{\fsize}{#1\fsize}\selectfont}%
  \ifx\svgwidth\undefined%
    \setlength{\unitlength}{378.74929353bp}%
    \ifx\svgscale\undefined%
      \relax%
    \else%
      \setlength{\unitlength}{\unitlength * \real{\svgscale}}%
    \fi%
  \else%
    \setlength{\unitlength}{\svgwidth}%
  \fi%
  \global\let\svgwidth\undefined%
  \global\let\svgscale\undefined%
  \makeatother%
  \begin{picture}(1,1.45011168)%
    \lineheight{1}%
    \setlength\tabcolsep{0pt}%
    \put(0,0){\includegraphics[width=\unitlength,page=1]{imgs/lattice.pdf}}%
    \put(0.35046471,0.00112154){\color[rgb]{0,0,0}\makebox(0,0)[lt]{\lineheight{1.25}\smash{\begin{tabular}[t]{l}false\end{tabular}}}}%
    \put(0.39911013,1.39449586){\color[rgb]{0,0,0}\makebox(0,0)[lt]{\lineheight{1.25}\smash{\begin{tabular}[t]{l}true\end{tabular}}}}%
    \put(0,0){\includegraphics[width=\unitlength,page=2]{imgs/lattice.pdf}}%
    \put(0.83466728,0.92664009){\color[rgb]{0,0,0}\makebox(0,0)[lt]{\lineheight{1.25}\smash{\begin{tabular}[t]{l}anti-chain\end{tabular}}}}%
    \put(0.42122763,0.76861602){\color[rgb]{0,0,0}\makebox(0,0)[lt]{\lineheight{1.25}\smash{\begin{tabular}[t]{l}$\varphi_3$\end{tabular}}}}%
    \put(0.08600779,0.76578717){\color[rgb]{0,0,0}\makebox(0,0)[lt]{\lineheight{1.25}\smash{\begin{tabular}[t]{l}$\varphi_2$\end{tabular}}}}%
    \put(0.75248513,0.77059602){\color[rgb]{0,0,0}\makebox(0,0)[lt]{\lineheight{1.25}\smash{\begin{tabular}[t]{l}$\varphi_4$\end{tabular}}}}%
    \put(0.42264212,0.29251887){\color[rgb]{0,0,0}\makebox(0,0)[lt]{\lineheight{1.25}\smash{\begin{tabular}[t]{l}$\varphi_0$\end{tabular}}}}%
    \put(0.64046429,1.14117689){\color[rgb]{0,0,0}\makebox(0,0)[lt]{\lineheight{1.25}\smash{\begin{tabular}[t]{l}$\varphi_6$\end{tabular}}}}%
    \put(0.27271256,1.1425912){\color[rgb]{0,0,0}\makebox(0,0)[lt]{\lineheight{1.25}\smash{\begin{tabular}[t]{l}$\varphi_7$\end{tabular}}}}%
  \end{picture}%
\endgroup%

%% file: sections/experiments.tex
\section{Experiments and Discussion}\label{sec:casestudies}
\mypara{Scenario}
Recall our introductory gridworld example Ex~\ref{ex:intro}.
Now imagine that the robot is pre-programmed to perform task the
``recharge and avoid lava'' task, but is unaware of the second
requirement, ``do not recharge when wet''. To signal this additional
constraint to the robot, the human operator provides the five
demonstrations shown in Fig~\ref{fig:gridworld}. We now
illustrate how learning specifications rather than Markovian rewards
enables the robot to safely compose the new constraint with
its existing knowledge to perform the joint task in a manner that
is robust to changes in the task. 

To begin, we assume the robot has access to the Boolean features: red
(lava tile), blue (water tile), brown (drying tile), and yellow
(recharge tile). Using these features, the robot has encoded the
``recharge and avoid lava'' task as: $H(\neg \text{red}) \wedge P(\text{yellow})$.

\thiswillnotshow{
\begin{figure}[h]
  \centering
  \includegraphics[height=1.5in]{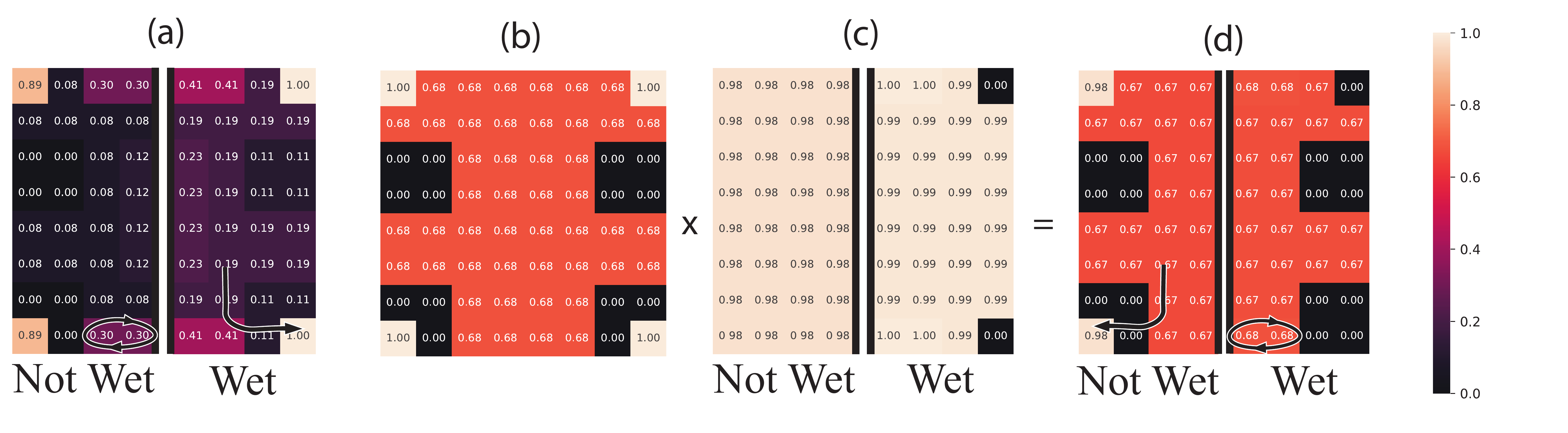}
  \caption{
    {\footnotesize {
    Learned rewards on the task discussed in Ex~\ref{ex:intro}. 
    (a) Rewards for the task YR \emph{and} BBY. Arrows indicate how agents who are trying to maximize reward would behave (left is a myopic agent; right is non-myopic agent).
    (b) Rewards for the sub-task YR.
    (c) Rewards for the sub-task BBY.
    (d) Rewards when YR and BBY are combined multiplicatively. 
    (a)-(c) Used MaxEnt IRL to infer rewards. In (a), (c), and (d) because
    the gridworld is symmetric, we use the left (resp. right) side to show the
    case where the robot is not wet (wet). Colors indicate magnitude of reward at that
  location, where lighter indicates higher reward.
  }}
    \label{fig:irl}
  }
\end{figure}

\emph{MaxEnt:} We began by applying MaxEnt on all of the features to
learn the monolithic task.  The resulting reward structure is shown in
Fig~\ref{fig:irl}a. However, as Fig~\ref{fig:irl}a illustrates, a
myopic agent is not incentivized to make progress towards the
recharging stations and a non-myopic agent will enter the lava to
reach the recharge station at the bottom. Next, we retrained on
subsets of the features to learn the sub-tasks (see Fig~\ref{fig:irl}b
and Fig~\ref{fig:irl}c). Observe that while the two reward
structures encode YR and BBY individually, it is not obvious how to
compose the rewards to jointly encode YR \emph{and} BBY.  For example,
just averaging makes the agent no longer avoid the lava. Taking the
minimum means that if the agent becomes wet it loses its incentive to
dry off or visit the yellow. Multiplication almost works, but suffers
from similar bugs as the monolithically learned reward (see Fig~\ref{fig:irl}d). 
}

\begin{minipage}{0.97\linewidth}
  \begin{minipage}{0.5\linewidth}
    \mypara{Concept Class} We designed the robot's concept class to be
    the conjunction of the known requirements and a specification
    generated by the grammar on the right.  The motivation in choosing
    this grammar was that (i) it generates a moderately large
    concept class (930 possible specifications after pruning trivially false specifications), and (ii) it contains
    several interesting alternative tasks such as
    $H(\text{red} \implies (\neg\text{brown}~S~\text{blue}))$, which semantically translates to: ``the robot should be wet before
entering lava''. To generate
  \end{minipage}
  \hfill
  \fbox{
    \begin{minipage}{0.45\linewidth}
      \emph{Concept Class Grammar:}
    \begin{bnf*}
      \bnfprod{$\phi$}
      {\bnfpn{H\bnfsp$\psi$} \bnfor \bnfpn{P\bnfsp$\psi$}}\\
      \bnfprod{$\psi$}
      {\bnfpn{$\beta$} \bnfor \bnfpn{$\beta$}
        \bnfts{$\implies$} \bnfpn{$\beta$}}\\
      \bnfprod{$\beta$}
      {\bnfpn{$\alpha$} \bnfor \bnfpn{$\alpha$}
        \bnfsp\bnfts{$\wedge$}\bnfsp\bnfpn{$\alpha$} \bnfor \bnfpn{$\alpha$}
        \bnfsp\bnfts{S}\bnfsp \bnfpn{$\alpha$}}\\
      \bnfprod{$\alpha$}
      {\bnftd{AP} \bnfor \neg\bnftd{AP}}\\
      \bnfprod{AP}
      { \bnftd{yellow} \bnfor \bnftd{red} \bnfor \bnftd{brown} \bnfor \bnftd{blue}}
    \end{bnf*}      
  \end{minipage}
  }
\end{minipage}

\vspace{-3px}
 the edges in Hasse diagram, we unrolled
the formula into their corresponding Boolean formula and used a SAT
solver to determine subset relations. While potentially slow, we make
three observations regarding this process: (i) the process was
trivially parallelizable (ii) so long as the atomic predicates remain
the same, this Hasse diagram need not be recomputed since it is
otherwise invariant to the dynamics (iii) most of the edges in the
resulting diagram could have been syntactically identified 
using well known identities on temporal logic formula.

\mypara{Computing $\widetilde{\varphi}$} To perform random satisfaction
rate queries, $\widetilde{\varphi}$, we first ran Monte Carlo to get a
coarse estimate and  we symbolically encoded the
dynamics, color sensor, and specification into a Binary Decision
Diagram to get exact values. This data structure serves as an incredibly succinct encoding
of the specification aware unrolling of the dynamics, which in practice
avoids the exponential blow up suggested by the curse of history.
We then counted the number of satisfying assignments and divided by
the total possible number of satisfying assignments.\footnote{One can
add probabilities to transitions by adding to transition constraints
additional fresh variables such that the number of satisfying
assignments is proportional to the probability.  } On average in these
candidate pools, each query took $0.4$ seconds with a standard
deviation of $0.32$ seconds.

\mypara{Results}
Running a fairly unoptimized implementation of
Algorithm~\ref{alg:infer_po} on the concept class and demonstrations
took approximately 95 seconds and resulted in 172 $\widetilde{\varphi}$
queries ($\approx 18\%$ of the concept class).  The inferred
additional requirement was $H((\text{yellow} \wedge P~\text{blue})
\implies (\neg \text{blue}~S~\text{brown}))$ which exactly captures
the do not recharge while wet constraint.  Compared to a brute force
search over the concept class, our algorithm offered an approximately
5.5 fold improvement. Crucially, there exists controllable
trajectories satisfying the joint specification:
\begin{equation} \bigg(H \neg \text{red} \wedge P~\text{yellow}\bigg)
\wedge H\bigg((\text{yellow} \wedge P~\text{blue}) \implies (\neg
\text{blue}~S~\text{brown})\bigg).
\end{equation}  Thus, a specification optimizing agent \emph{must}
jointly perform both tasks. This holds true even under task changes
such as that in Fig~\ref{fig:brittle}. Further, observe that it was
fairly painless to incorporate the previously known recharge while
avoiding lava constraints. Thus, in contrast to quantitative Markovian
rewards, learning Boolean specifications enabled encoding
\textit{compositional} \textit{temporal specifications} that are
\textit{robust} to changes in the environment.


%% file: sections/conclusion.tex
\section{Conclusion and Future work}
\vspace{-10px}
Motivated by the problem of compositionally learning from
demonstrations, we developed a technique for learning binary
non-Markovian rewards, which we referred to as \emph{specifications}.
Because of their limited structure, specifications enabled first
learning sub-specifications for subtasks and then later creating a
composite specifications that encodes the larger task.  To learn these
specifications from demonstrations, we applied the principle of
maximum entropy to derive a novel model for the likelihood of a
specification given the demonstrations. We then developed an algorithm
to efficiently search for the most probable specification in a
candidate pool of specifications in which some subset relations
between specifications are known. Finally, in our experiment, we gave
a concrete instance where using traditional learning composite reward
functions is non-obvious and error-prone, but inferring specifications
enables trivial composition. Future work includes extending the
formalism to infinite horizon specifications, continuous dynamics,
characterizing the optimal set of teacher demonstrations under our
posterior model~\cite{teachingCphs}, efficiently marginalizing over the whole concept
class and exploring alternative data driven methods for generating
concept classes.
